\documentclass[twoside,11pt]{article}
\pdfoutput=1
\usepackage{jmlr2e}
\usepackage{xstring}
\usepackage{amsmath}
\usepackage{bm}
\usepackage{mathtools}
\usepackage{color}
\usepackage{nicefrac}
\usepackage{bbm}

\usepackage{enumitem}
\usepackage{ragged2e}

\usepackage{algorithm}
\usepackage{algorithmic}

\usepackage{tabularx}
\usepackage{multirow}
\usepackage{wrapfig} 
\usepackage[latin1]{inputenc}
\usepackage{tikz}

\newenvironment{proof2}{\par\noindent}{\hfill\BlackBox\\[.3mm]}

\newcommand{\mc}[1]{\ensuremath \mathcal{#1}}
\newcommand{\mb}[1]{\ensuremath \mathbb{#1}}

\DeclarePairedDelimiter{\braces}{\{}{\}}
\DeclarePairedDelimiter{\paren}{\lparen}{\rparen}

\newcommand{\xdim}{d}
\newcommand{\canon}{max-canonical form}
\newcommand{\pzero}{\mc{P}_0}
\newcommand{\pzeroall}{\pzero^{\text{all}}}
\newcommand{\qzero}{{\mathcal{Q}_0}}
\newcommand{\alphamax}{{\alpha^*}}
\newcommand{\ifdir}{``\!\!\Longleftarrow\!\!"}
\newcommand{\onlyifdir}{``\!\!\Longrightarrow\!\!"}

\newcommand{\RR}{\mathbb{R}}
\newcommand{\betavec}{\boldsymbol{\beta}}

\newcommand{\nkernels}{k}

\newcommand{\xset}{\mathcal{X}}
\newcommand{\xsett}{\mathcal{X}_\tau}
\newcommand{\yset}{\mathcal{Y}}
\newcommand{\B}[1]{\mc{B}(\IfEqCase{#1}{
                    {1}{\mb{R}}}
                    [\mb{R}^{#1}])}
\newcommand{\Rd}[1]{(\IfEqCase{#1}{
                    {1}{\mb{R}}}
                    [\mb{R}^{#1}],\B{#1})}
\newcommand{\Salgebra}{\mc{A}}

\newcommand{\cone}{f_1}
\newcommand{\czero}{f_0}

\newcommand{\newcone}{h_1}
\newcommand{\newczero}{h_0}
\newcommand{\mixf}{f}

\newcommand{\mua}{\lambda_0}
\newcommand{\mub}{\nu_0}
\newcommand{\mumax}{\mu_0^{*}}

\newcommand{\alphamaxtau}{{\alpha_{\tau}^{*}}}
\newcommand{\alphamaxmax}{\alpha^{**}}

\newcommand{\taudiv}{\tau_d}
\newcommand{\mumix}{\mu}
\newcommand{\muone}{\mu_1}
\newcommand{\muzero}{\mu_0}
\newcommand{\numix}{\nu}
\newcommand{\nuone}{\nu_1}
\newcommand{\lammix}{\lambda}
\newcommand{\lamone}{\lambda_1}
\newcommand{\nuzero}{\nu_0}
\newcommand{\leb}{\mathbb{L}}

\newcommand{\elkan}{Elkan-Noto}
\newcommand{\pdfratio}{pdf ratio}
\newcommand{\cdfratio}{cdf based}

\newcommand{\AlgName}{AlphaMax}
\newcommand{\mixwt}{w}

\newcommand*{\argmax}{\mathop{\mathrm{argmax}}}

\newcommand{\Aepsilon}{A_\epsilon}
\newcommand{\Bepsilon}{B_\epsilon}
\newcommand{\Tepsilon}{T_\epsilon}
\newcommand{\Tepsiloni}{T_{1\epsilon}}

\jmlrheading{}{2015}{}{12/15}{00/00}{Shantanu Jain, Martha White, Michael W. Trosset and Predrag Radivojac}

\ShortHeadings{Nonparametric semi-supervised learning of class proportions}{Jain, White, Trosset and Radivojac}
\firstpageno{1}

\begin{document}

\title{Nonparametric semi-supervised learning of class proportions}


\author{\name Shantanu Jain \email shajain@indiana.edu \\
       \addr Department of Computer Science and Informatics\\
       Indiana University\\
       Bloomington, IN 47405, USA
       \AND
       \name Martha White \email martha@indiana.edu \\
       \addr Department of Computer Science and Informatics\\
       Indiana University\\
       Bloomington, IN 47405, USA
       \AND
       \name Michael W. Trosset \email mtrosset@indiana.edu \\
       \addr Department of Statistics\\
       Indiana University\\
       Bloomington, IN 47408, USA
       \AND
       \name Predrag Radivojac \email predrag@indiana.edu \\
       \addr Department of Computer Science and Informatics\\
       Indiana University\\
       Bloomington, IN 47405, USA}

\editor{Leslie Pack Kaelbling}

\maketitle

\begin{abstract}
The problem of developing binary classifiers from positive and unlabeled data is often encountered in machine learning. 
A common requirement in this setting is to approximate posterior probabilities of positive and negative classes for a previously unseen data point. This problem can be decomposed into two steps: (i) the development of accurate predictors that discriminate between positive and unlabeled data, and (ii) the accurate estimation of the prior probabilities of positive and negative examples. In this work we primarily focus on the latter subproblem. We study nonparametric class prior estimation and formulate this problem as an estimation of mixing proportions in two-component mixture models, given a sample from one of the components and another sample from the mixture itself. We show that estimation of mixing proportions is generally ill-defined and propose a canonical form to obtain identifiability while maintaining the flexibility to model any distribution. We use insights from this theory to elucidate the optimization surface of the class priors and propose an algorithm for estimating them. To address the problems of high-dimensional density estimation, we provide practical transformations to low-dimensional spaces that preserve class priors. Finally, we demonstrate the efficacy of our method on univariate and multivariate data.
\end{abstract}


\section{Introduction}
\label{sec:intro}

Binary classification is often attempted in situations where the examples from one class greatly outnumber the examples from the other class \citep{Chawla2004}. An extreme case of this scenario occurs when the examples of one class (say, positives) are relatively easy to obtain, while the examples of the other class (say, negatives) are either too expensive or practically impossible to obtain. In such problems we are often presented with data sets containing a relatively small number of positive examples and a relatively large number of unlabeled examples that contain both positive and negative examples at unknown proportions. 

Positive and unlabeled data sets are often observed in the sciences, where the absence of a class designation, even after a failure to label the data point as positive, cannot be interpreted as a negative class label. For example, a protein can be experimentally tested for a particular functionality; e.g., catalytic activity. If confirmed, the data associating a protein with catalytic activity is reliable; however, a failure to confirm catalytic activity may only be due to experimental limitations \citep{Dessimoz2013}. Further compounding the problem, unsuccessful experiments are rarely reported in the literature. Another situation conforming to the open-world assumption occurs in social networks, where a click on the ``like" button on Facebook is a reliable indicator of preference, yet the absence of a ``like" cannot be considered as an indicator of dislike. A simple and important question in all such situations it that of estimating class priors: What is the fraction of all proteins in a given species that are enzymes? or How many Facebook users like a particular product?

The positive-unlabeled data sets do not conform to typical assumptions in machine learning. Traditional supervised algorithms assume the existence of both positive and negative examples, whereas unsupervised algorithms operate without any information on class labels. Even most semi-supervised algorithms assume both positive and negative class labels, in addition to the unlabeled set, and thus require modification. This framework has been studied in the past decade and a half, usually under the names of partially supervised learning and learning from positive and unlabeled data \citep{Liu2002, Denis2005}. Regardless of the problem type, the main goal in all these approaches is to learn classifiers that discriminate between positives and negatives using the available data, ideally by estimating posterior probabilities of class labels given an input example. Here, we are (scientifically) motivated by the problem of estimating class priors: the proportions of positive and negative examples in the unlabeled data, given a set of positive examples and an unlabeled set. 

More formally, we consider the binary classification problem of mapping an input space $\xset$ to an output space $\yset =\{0,1\}$ given i.i.d.~samples of positive and unlabeled examples from fixed but unknown underlying probability distributions. We formulate estimation of the fraction of positive examples in the unlabeled data as parameter learning of two-component mixture models
\begin{equation}
f(x)=\alpha  f_{1}(x) + (1-\alpha)  f_{0}(x), \label{eq_mainproblem}
\end{equation}
\noindent where $x \in \xset$, $f_{1}$ and $f_{0}$ are distributions of the positive and negative data, respectively, and $\alpha \in (0,1)$ is the \emph{mixing proportion}\/ or \emph{class prior}\/ for the positive examples. In the simplest setting, $f_{1}(x)$ and $f_{0}(x)$ correspond to the class-conditional distributions $p(x | y=1)$ and $p(x | y=0)$, respectively. More generally, $f_{1}$ and $f_{0}$ might be any distributions obtained after applying deterministic transformations that preserve the mixing proportions. For example, a function $g:\xset \rightarrow [0,1]$, such as a classifier trained on the labeled vs.\ unlabeled data, could map the original feature vector $x$ to a scalar, resulting in a univariate $f$. We will later discuss class-prior preserving transformations; in the meantime, we refer to $f_{1}$, $f_{0}$, and $f$ as data distributions.

Despite the simplicity of our formulation, we are not aware of any previous attempts to formulate class prior estimation using two-component mixture models. We therefore begin by discussing identifiability conditions for mixing proportions.  Using insights derived from this theory, we then propose algorithms for learning $\alpha$ using maximum-likelihood estimation.  Finally, we conduct experiments on both synthetic and real-world data, obtaining evidence that our methodology is sound and effective.

\section{Identifiability issues and the max-canonical form}
\label{sec:iden}

If the mixture distribution can be written as (\ref{eq_mainproblem}) for more than one choice of $\alpha$, then estimation of $\alpha$ is ill-defined.
In order to proceed, we require
\textit{identifiability}: the existence of a unique mixing proportion. 
In general, mixing proportions are not identifiable. In this section we propose a \canon \ for $\alpha$ that ensures identifiability and leads to a viable estimation algorithm. 


To formalize identifiability of the mixing proportion, we more generally use probability measures to enable discrete, continuous and mixed random variables to be considered. Let $\mumix$, $\muone$ and $\muzero$ be probability measures defined on a $\sigma$-algebra $\Salgebra$ and let $$\mumix(A)=\alpha \muone(A) + (1-\alpha) \muzero(A),$$ for all $A \in \Salgebra$. This expression generalizes the mixture model from \autoref{eq_mainproblem} to probability measures. For example, for a continuous random variable with probability density function $f(x)$, the associated measure for any $A \in \Salgebra$ is $\mu(A) = \int_A f(x) dx$.\footnote{More technically, $\mu(A) = \int_A f(x) \leb(dx)$ for all $A \in \B{\xdim}$, where $\B{\xdim}$ is the Borel $\sigma$-algebra on $\RR^\xdim$ and $\leb$ is the Lebesgue measure on $\RR^\xdim$.} Let now $\pzero$ be an arbitrary family of probability measures defined on $\Salgebra$ such that $\pzero \cap \{\muone \} = \emptyset$. Finally, let us define a family of non-trivial two-component mixtures as 
\begin{align}
& \mc{F}(\pzero,\muone) =  \{ \alpha \muone + (1-\alpha)\muzero : \muzero \in \pzero, \ \alpha \in (0,1)\}  \label{eq:mixMod}.
\end{align}

\noindent For $\mc{F}(\pzero, \muone)$ to be identifiable, there needs to be a one-to-one mapping between the set of parameters $(\alpha, \muzero) \in (0,1) \times \pzero$ and $\mumix \in \mc{F}(\pzero, \muone)$; that is, each $\mumix$ must correspond to a unique set of parameters. We formalize this in the following definition and Lemma \ref{lem:twoIden}. 
%
\begin{definition}[Identifiability]
For any non-empty set of probability measures $\pzero$ and any probability measure $\muone$ on $\Salgebra$ such that $\pzero \cap \{\muone\} = \emptyset$, $\mc{F}(\pzero,\muone)$ is said to be \emph{identifiable} if $\forall a,b \in (0,1)$ and $\forall \mua,\mub \in \pzero$
\begin{align*}
a\muone + (1-a)\mua = b\muone + (1-b)\mub  \ \ \ \implies \ \ \ (a,\mua)=(b,\mub)
.
\end{align*}
Similarly, $\mc{F}(\pzero,\muone)$ is said to be \emph{identifiable in $\alpha$}
if $\forall a,b \in (0,1)$ and $\forall \mua,\mub \in \pzero,$
\begin{align*}
a\muone + (1-a)\mua = b\muone + (1-b)\mub \ \ \ \implies \ \ \  a=b
.
\end{align*}
 \end{definition}
Fortunately, using the following lemma, we can focus on identifiability in terms of the mixing proportion, which is the parameter we are interested in estimating. Therefore, after Lemma \ref{lem:twoIden} we will drop the qualification ``in $\alpha$" when discussing identifiability of $\mc{F}(\pzero,\muone)$
\begin{lemma}\label{lem:twoIden}
$\mc{F}(\pzero,\muone)$ is identifiable if and only if $\mc{F}(\pzero,\muone)$ is identifiable in $\alpha$.
\end{lemma}
\begin{proof}
If $\mc{F}(\pzero,\muone)$ is identifiable, then $\forall a,b \in (0,1)$ and $\forall \mua,\mub \in \pzero$ $$a\muone + (1-a)\mua = b\muone + (1-b)\mub \Rightarrow (a,\mua)=(b,\mub) \Rightarrow a = b$$ and so $\mc{F}(\pzero,\muone)$ is identifiable in $\alpha$.

If $\mc{F}(\pzero,\muone)$ is identifiable in $\alpha$, then for any $a,b \in (0,1)$ and $\mua,\mub \in \pzero$ \text{ where } $a\muone + (1-a)\mua = b\muone + (1-b)\mub$, we know $a = b$. Replacing $b$ with $a$, we obtain
\begin{align*}
a\muone +(1-a)\mua = b\muone + (1-b)\mub \ \ \ 
&\implies  a\muone +(1-a)\mua = a\muone + (1-a)\mub\\
      &\implies \mua=\mub
\end{align*}
Therefore, $(a,\mua)=(b,\mub)$ and, by definition, $\mc{F}(\pzero,\muone)$ is identifiable.
\end{proof}
Unfortunately, in general, $\mc{F}(\pzero,\muone)$ is not identifiable. For $\mumix \in \mc{F}(\pzero,\muone)$, we define the following set
\begin{align}
A(\mumix,\muone,\pzero)=\braces*{\alpha \in (0,1):\mumix=\alpha\muone + (1-\alpha)\muzero, \ \textrm{where}\  \muzero \in \pzero}
.
\end{align}
For $\mc{F}(\pzero,\muone)$ to be identifiable, $A(\mumix,\muone,\pzero)$ must be a singleton for every $\mumix \in \mc{F}(\pzero,\muone)$. In Lemma \ref{lem:Ah}, we demonstrate that this is not the case for $\pzeroall$, the set of all measures on $\Salgebra$ except $\muone$; in fact, we will see that every mixture in $\mc{F}(\pzeroall,\muone)$ corresponds to an entire interval of choices for $\alpha$. We next present Theorem \ref{thm:famQ0}, which enables us to remedy this non-identifiability issue by restricting $\muzero$ to a smaller family, without losing the modeling flexibility of $\mc{F}(\pzeroall,\muone)$. 
We will first discuss the ramifications of this theorem and our proposed max-canonical form; we prove the theorem and required lemmas at the end of the section. Table \ref{tab:notation} provides the summary of notation for quick reference.

\begin{table}[t]
\centering
\caption{A summary of notation.}
\footnotesize
\begin{tabular}{|c|c|}
\hline 
Symbol & Definition\\
\hline 
\hline 
$\mu$, $\mu_{1}$, $\mu_{0}$ & Probability measures: mixture ($\mu$), component one ($\mu_{1}$), component zero ($\mu_{0}$)\\
$f$, $f_{1}$, $f_{0}$ & Probability density functions: mixture ($f$), component one ($f_{1}$), component zero ($f_{0}$)\\
$\alpha$ & Mixing proportion, class prior\\
$\mathcal{A}$ & Sigma algebra\\
$\mathcal{P}_{0}$ & An arbitrary family of distributions from which $\mu_{0}$ is selected.
$\mathcal{P}_{0}\cap\{\mu_{1}\}=\emptyset$\\
$\mathcal{P}_{0}^{\textrm{all}}$ & The family of all distributions, except $\mu_{1}$, on $\mathcal{A}$ from which $\mu_{0}$ is selected. $\mathcal{P}_{0}^{\textrm{all}}\cap\{\mu_{1}\}=\emptyset$\\
$\mathcal{F}(\mathcal{P}_{0},\mu_{1})$ & $\mathcal{F}(\mathcal{P}_{0},\mu_{1})=\left\{ \alpha\mu_{1}+(1-\alpha)\mu_{0} : \mu_{0}\in\mathcal{P}_{0},\:\alpha\in(0,1)\right\} $\\
$\mathcal{Q}_{0}$ & $\mathcal{Q}_{0}=\mathcal{P}_{0}^{\textrm{all}}\setminus\mathcal{F}(\mathcal{P}_{0},\mu_{1})$\\
$A(\mu,\mu_{1},\mathcal{P}_{0})$ & $A(\mu,\mu_{1},\mathcal{P}_{0})=\{\alpha\in(0,1):\mu=\alpha\mu_{1}+(1-\alpha)\mu_{0},\:\textrm{where}\:\mu_{0}\in\mathcal{P}_{0}\}$\\
$R(\mu,\mu_{1})$ & $R(\mu,\mu_{1})=\{\nicefrac{\mu(A)}{\mu_{1}(A)}:A\in\mathcal{A},\:\mu_{1}(A)>0\}$\\
$\alpha^{*}$ & $\alpha^{*}=\inf R(\mu,\mu_{1})$; $A(\mu,\mu_{1},\mathcal{P}_{0}^{\textrm{all}}) = (0, \alpha^{*}]$\\
$\mu_{0}^{*}$ & The particular $\mu_{0}$ when $\alpha = \alpha^{*}$\\
\hline 
\end{tabular}
\normalsize
\label{tab:notation}
\end{table}

\begin{theorem} \label{thm:famQ0}
For $\qzero= \pzeroall \setminus \mc{F}(\pzeroall,\muone)$ the following hold: 
\begin{enumerate}
\item $\mc{F}(\qzero,\muone)$ is identifiable and contains the same mixtures as $\mc{F}(\pzeroall,\muone)$.
\item For any $ \mumix \in \mc{F}(\pzeroall,\muone)$, there exists $\alphamax \in (0,1)$ such that $A(\mumix,\muone,\pzeroall) = (0, \alphamax]$ and there exists $\mumax  \in \qzero$ such that $\mumix = \alphamax \muone  +  (1-\alphamax) \mumax$.
\item For any $\alpha < \alphamax$ with $\muzero$ such that $\mumix = \alpha \muone + (1-\alpha) \muzero$, $\muzero$ is a non-trivial mixture containing $\muone$ and the distance between $\alpha$ and the upper bound $\alphamax$ is smoothly 
related to the proportion of $\muone$ in $\muzero$; that is, $\muzero \in \mc{F}(\pzeroall,\muone)$ and $\alphamax - \alpha =  (1-\alpha) \max A(\muzero,\muone, \pzeroall)$.
\end{enumerate}
\end{theorem}
Interestingly, the two families $\mc{F}(\qzero,\muone)$ and $\mc{F}(\pzeroall,\muone)$ contain the same mixtures, yet $\mc{F}(\qzero,\muone)$ is identifiable and $\mc{F}(\pzeroall,\muone)$ is not. 
Importantly, \autoref{thm:famQ0} suggests a canonical form for the estimation of the mixing proportion that ensures identifiability by selecting $\alphamax = \max A(\mumix,\muone,\pzeroall)$. With this \textit{\canon}, estimation is implicitly restricted to the identifiable set $\mc{F}(\qzero,\muone)$, while maintaining the ability to model any mixture in $\mc{F}(\pzeroall,\muone)$. This canonical form is intuitive in that it prefers $\muzero$ that is not composed of $\muone$. The reason for the lack of identifiability is that we can always shift some portion of $\muone$ into $\muzero$ until all weight is on $\muzero$ (i.e., $\alpha = 0$). Therefore, choosing the maximum $\alpha$ selects the most separated $\muzero$ and $\muone$. Moreover, statement 3 indicates that even if in practice $\muzero$ really does have a small proportion of $\muone$, the over-estimate of $\alpha$ smoothly relates to this small proportion.


\begin{figure}[]
  \centering
  \includegraphics[scale=1.2]{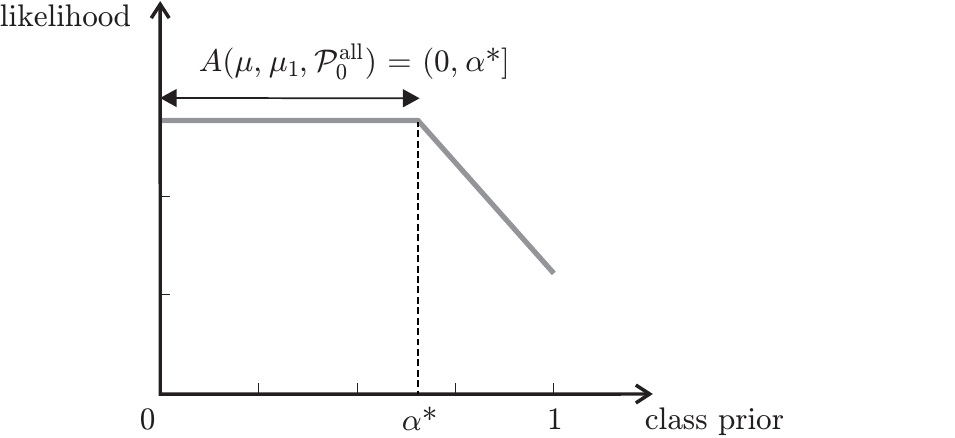}
\caption{Stylized depiction of the likelihood as a function of the mixing proportion. The true mixing proportion is not identifiable and lies in the region $A(\mumix,\muone,\pzeroall)$. The end of the interval $\alphamax$ is identifiable and preserves representation of $\mc{F}(\pzeroall,\muone)$. A procedure estimating the likelihood is expected to show flat likelihood for any $\alpha \in A(\mumix,\muone,\pzeroall)$ and then decrease after $\alphamax$.
  }
  \label{fig:alpha}
\end{figure}

\citet{Blanchard2010} give an identifiability result that captures several aspects of Theorem \ref{thm:famQ0}; precisely, in identifying $\qzero$'s role towards identifiability and establishing $\alphamax$ as the maximum mixing proportion. Our theorem additionally shows that there is no loss in the modeling flexibility by restricting to $\qzero$ and that the set of all valid mixing proportions is actually the interval $(0,\alphamax]$. This interval is significant because it directly informs our algorithm development in Section \ref{sec:algo}. Our theorem also quantifies the error of $\alphamax$ when $\muzero \notin \qzero$. 

Non-identifiability of the mixing proportion suggests an interval of equally likely solutions using inference techniques such as maximum likelihood, where $\muone$ and $\mumix$ are approximated using (integrals of) density or mass functions. The approximation is expected to progressively deteriorate as the mixing proportion increases beyond $\alphamax$. Figure \ref{fig:alpha} illustrates the expected relationship between a likelihood function and the mixing proportion. 

\subsection{Proof of Theorem \ref{thm:famQ0}}

We first prove the following two lemmas. 
%

\begin{lemma}
\label{lem:Ah}
For any $\mumix \in \mc{F}(\pzeroall,\muone)$, let
$$R(\mumix,\muone)=\braces*{\nicefrac{\mumix(A)}{\muone(A)}: A \in \Salgebra, \ \muone(A)>0}.$$
Then, there exists 
\begin{align*}
\alphamax=\inf R(\mumix,\muone) \in (0, 1) \hspace{1.0cm} \text{ and } \hspace{1.0cm} A(\mumix,\muone,\pzeroall) = (0,\alphamax]
.
\end{align*}
Hence, $\mc{F}(\pzeroall,\muone)$ is nonidentifiable. Moreover, $\alphamax$ can be defined in terms of the densities $\mixf,\cone$ corresponding to $\mumix,\muone$, if they exist; precisely, $\alphamax = \inf \braces*{\nicefrac{\mixf(x)}{\cone(x)}: x \in \xset, \ \cone(x)>0}.$

\end{lemma}
\begin{proof}

\noindent
\textbf{Part 1:}
First, we show that $\alphamax$ is well defined and it is in $(0,1)$. $R(\mumix,\muone)$ is non-empty, because there exists $A\in \Salgebra$ with $\muone(A)>0$; thus, $\alphamax$ is well defined. Now, because $\mumix\in \mc{F}(\pzeroall,\muone)$, there exists $b \in (0,1)$ and $\mua\in \pzeroall$ such that $\mumix=b\muone + (1-b)\mua$. For any $A \in \Salgebra$ such that $\muone(A) > 0$,
\begin{align*}
    \frac{\mumix(A)}{\muone(A)} &= \frac{b\muone(A) + (1-b)\mua(A)}{\muone(A)} 
    =  b+ (1-b)\frac{\mua(A)}{\muone(A)}
    \geq b 
\end{align*}
giving $\alphamax \ge b > 0$.
To show that $\alphamax < 1$,
suppose there exists $\mumix$ and $\muone$ such 
that $\alphamax \geq 1$. Then because $\alphamax$ is the infimum of $R(\mumix,\muone)$
$$\frac{\mumix(A)}{\muone(A)} \geq 1 \text{ when $\muone(A)  > 0$}.$$
Moreover, when $\muone(A) = 0$, clearly $\mumix(A) \ge \muone(A)$. 
Therefore, $\mumix(A) \ge \muone(A)$ for all $A \in \Salgebra$.
Now because $\muone \notin \mc{F}(\pzeroall,\muone)$,  $\mumix \neq \muone$
and so there exists $B \in \Salgebra$ such that $\mumix(B) > \muone(B)$.
This leads to a contradiction as follows:
\begin{align*}
1 = \mumix(\xset)  &= \mumix(\xset \setminus B) + \mumix(B)\\
&> \muone(\xset \setminus B) + \muone(B)\\
&= 1
\end{align*}
Thus $\alphamax < 1$ and $\alphamax \in (0,1)$. 
   
\textbf{Part 2:} Second, we show that $A(\mumix,\muone,\pzeroall) \supseteq (0,\alphamax]$.
Recall that $A(\mumix,\muone,\pzeroall) = \braces*{\alpha \in (0,1):\mumix=\alpha\muone + (1-\alpha)\muzero, \ \textrm{where}\  \muzero \in \pzeroall}$.
We need to show that if $a \in (0, \alphamax]$, then there exists $\mua \in \pzeroall$ such that $\mumix=a\muone + (1-a)\mua$. Consider
$$\mua = \frac{\mumix -a \muone}{1 -a}.$$
Because $a \leq \alphamax <1$, $\mua$ is well defined. Moreover, $\mua \neq \muone$ because equality would imply $\mumix = \muone$ (trivial mixture). If we can show $\mua$ is a probability measure, then $\mua\in \pzeroall$ and Part 2 is complete.

$\mua(A) \ge 0 \ \forall A \in \Salgebra$: By definition of $\alphamax$, for all $A \in \Salgebra$ with $\muone(A)> 0$, we have $\mumix(A) \ge \alphamax \muone(A) \ge a \muone(A)$ because $a \in (0, \alphamax]$. The inequality is trivial when $\muone(A)=0$. Thus, $\mumix(A)-a \muone(A) \ge 0$ for all $A \in \Salgebra$; consequently, $\mua(A) \ge 0.$

$\mathbf{\mua(\xset) = 1}$:
Because $\mumix(\xset) = 1$ and $\muone(\xset) = 1$,
\begin{align*}
\mua(\xset) = \frac{\mumix(\xset) - a\muone(\xset)}{1-a} = \frac{1-a}{1-a} = 1
. 
\end{align*} 
Similarly, $\mua(\emptyset) = 0$ and $\mua(\cup A_i) = \sum_i \mua(A_i)$.
Therefore, $\mua$ is a probability measure.

\textbf{Part 3:} Third, we show that $A(\mumix,\muone,\pzeroall) \subseteq (0,\alphamax]$. 
Take any $a \in A(\mumix,\muone,\pzeroall)$ and corresponding 
$\mua \in \pzeroall$ such that $\mumix=a\muone + (1-a)\mua$.
For any $A \in \Salgebra$, we know $(1-a)\mua(A) \ge 0$
and so $\mumix(A)\geq a\muone(A)$.
Thus, for all $A \in \Salgebra$ with $\muone(A) > 0$, $\nicefrac{\mumix(A)}{\muone(A)} \geq a$ and consequently $a\leq \alphamax$. 

\vspace{0.2cm}
\noindent
In summary, $A(\mumix,\muone,\pzeroall)=(0,\alphamax]$ and it is not a singleton set for any $\mumix$. Therefore, we conclude that $\mc{F}(\pzeroall,\muone)$ is not identifiable.

\textbf{Part 4:} Fourth, we show that 
$$\czero= \frac{\mixf-\alpha \cone}{1-\alpha}$$
is a valid probability density, if and only if $\alpha \in (0, \alphamax]$

$\ifdir$ Any $\alpha \in A(\mumix,\muone,\pzeroall)=(0,\alphamax]$ is a valid mixture; i.e., there exists $\muzero \in \pzeroall$ such that $\mumix = \alpha \muone +(1-\alpha) \muzero$. Because $\muzero$ is a well defined measure and it can be expressed as $\muzero =\nicefrac{(\mumix-\alpha \muone)}{(1-\alpha)}$, its probability density, $\newczero$, can be defined in terms of the densities corresponding to $\mumix$ and $\muone$ as follows: $\newczero=\nicefrac{(\mixf - \alpha \cone)}{(1-\alpha)}$. Thus $\czero = \newczero$ and it is a well defined probability density.

$\onlyifdir$ We give a proof by contradiction. Suppose $\czero$ is a well defined probability density for some $\alpha > \alphamax$. Then $\muzero$, the probability measure corresponding $\czero$, can be expressed in terms of probability measures corresponding to densities $\mixf$ and $\cone$ as follows: $\muzero=\nicefrac{(\mumix-\alpha \muone)}{(1-\alpha)}$; consequently, $\mumix = \alpha \muone +(1-\alpha) \muzero$. Thus $\alpha$ is a valid mixing proportion, i.e., $\alpha \in A(\mumix,\muone,\pzeroall)=(0,\alphamax]$, which gives the contradiction. 

\textbf{Part 5:} Next, we show that $\alphamax = \inf \braces*{\nicefrac{\mixf(x)}{\cone(x)}: x \in \xset, \ \cone(x)>0}$. We give a proof by contradiction. Suppose $\alphamaxmax = \inf \braces*{\nicefrac{\mixf(x)}{\cone(x)}: x \in \xset, \ \cone(x)>0} \neq \alphamax$.

If $\alphamaxmax < \alphamax$: Using the definition of $\alphamaxmax$, there exists $x \in \xset$ with $\cone(x)>0$ and $\nicefrac{\mixf(x)}{\cone(x)}<\alphamax$. Thus $\mixf(x) -\alphamax \cone(x)<0$ and $\nicefrac{(\mixf-\alphamax \cone)}{(1-\alphamax)}$ is not a probability density function, which contradicts part 4.

If $\alphamaxmax > \alphamax$: Using the definition of $\alphamaxmax$, $\alphamaxmax \leq   \nicefrac{\mixf(x)}{\cone(x)}$ when $\cone(x) > 0$. Thus $\mixf(x)-\alphamaxmax \cone(x)\geq 0$ when $\cone(x) > 0$. The inequality is trivially true when $\cone(x)=0$.  Consequently $\czero(x)=\nicefrac{\paren*{\mixf(x)-\alphamaxmax \cone(x)}}{(1-\alphamaxmax)} \geq 0$ for all $x \in \xset$. Moreover, $\int_\xset \czero(x) dx =1$. Thus $\czero$ is a well defined probability density function, which contradicts part 4. \\
Thus $\alphamaxmax=\alphamax$.
\end{proof}
%
%
\begin{lemma}
		\label{lem:iden}			$\mc{F}(\pzero,\muone)$ is identifiable if and only if $\mc{F}(\pzero,\muone) \cap \pzero=\emptyset.$
\end{lemma}
\begin{proof}
        $\ifdir$
        We give a proof by contradiction. 
        Suppose $\mc{F}(\pzero,\muone) \cap \pzero=\emptyset$, but $\mc{F}(\pzero,\muone)$ is not identifiable. 
        Therefore, there exists $a, b \in (0,1)$ and $\mua ,\mub \in \mathcal{P}_0$ such that  $a \muone + (1-a)\mua = b \muone + (1-b)  \mub$, but $a\neq b$. 
        Without loss of generality we can assume $a>b$. We now have
		\begin{align*}
			a \muone &+ (1-a)\mua = b \muone + (1-b)  \mub \\
			&\Rightarrow \mub = \nicefrac{(a-b)}{(1-b)}\muone + \nicefrac{(1-a)}{(1-b)}\mua = \nicefrac{(a-b)}{(1-b)}\muone + \paren*{1-\nicefrac{(a-b)}{(1-b)}}\mua
		\end{align*}
		Because $\nicefrac{(a-b)}{(1-b)} \in (0,1)$ and $\mua \in \pzero$, it follows that $\mub \in \mc{F}(\pzero,\muone)$. Moreover, $\mub$ was picked from $\pzero$. 
		Therefore, $\mc{F}(\pzero,\muone)\cap \pzero$ contains $\mub$ and so is not empty,
		which is a contradiction. \\
		   $\onlyifdir$
        Again, we give a proof by contradiction. Suppose $\mc{F}(\pzero,\muone)$ is identifiable, 
        but $\mc{F}(\pzero,\muone) \cap \pzero \neq \emptyset$. Let $\mua \in \mc{F}(\pzero,\muone) \cap \pzero$. Because $\mua \in \mc{F}(\pzero,\muone)$, there exists $\mub \in \pzero$ and $c \in (0,1)$ such that $\mua =c\muone + (1-c)\mub$. Let $a \in (c,1)$ and $b=\nicefrac{(a-c)}{(1-c)}$. As $a,b \in (0,1)$ and $\mub,\mua \in \pzero$, it follows that $a\muone +(1-a)\mub,  b\muone + (1-b)\mua \in \mc{F}(\pzero,\muone)$. First, we can see that $a\muone +(1-a)\mub = b\muone + (1-b)\mua$ because
        \begin{align*}
            b\muone + (1-b)\mua
            &= b\muone + (1-b)(c\muone+(1-c)\mub)\\
            &= \paren*{b+(1-b)c}\muone + (1-b)(1-c)\mub\\
            &= \paren*{b(1-c)+c}\muone + \paren*{1-b(1-c)-c}\mub\\
            &= a\muone + (1-a)\mub &&\triangleright a=b(1-c)+c 
            .
        \end{align*}
Because $a=b(1-c) +c$ is a convex combination of $b$ and $1$, we know $a \in (b,1)$, giving $a>b$ and so $a \neq b$. Therefore $a\muone +(1-a)\mub = b\muone + (1-b)\mua$, but $(a,\mub)\neq (b,\mua)$. It follows $\mc{F}(\pzero,\muone)$ is not identifiable, which is a contradiction.
\end{proof}

%
\begin{proof2}[\textbf{Proof of Theorem 1:}]
\textbf{To prove Statement 2} we use Lemma \ref{lem:Ah}. 
For any $\mumix \in \mc{F}(\pzeroall,\muone)$,
 we know there exists an $\alphamax \in (0,1)$ such that $A(\mumix,\muone,\pzeroall) = (0,\alphamax]$.
 This means that $\alphamax$ is a valid mixing proportion for mixture $\mumix$. Thus, there exists $\mumax \in \pzeroall$ such that $\mumix = \alphamax \muone + (1-\alphamax)\mumax$. To show that $\mumax \in \qzero$, we need to show that $\mumax \notin \mc{F}(\pzeroall,\muone)$. 
Using a proof by contradiction, assume $\mumax \in \mc{F}(\pzeroall,\muone)$.
For some $a \in (0,1)$ and $\mua \in \pzeroall$, $\mumax=a\muone + (1-a) \mua$, giving  
\begin{align*}
\mumix&=\alphamax \muone + (1-\alphamax) (a\muone + (1-a) \mua)\\
&=\left(\alphamax+ (1 -\alphamax)a\right)\muone + (1-\alphamax)(1-a)\mub\\
&=\left(\alphamax+ (1 -\alphamax)a\right)\muone + (1-(\alphamax+ (1 -\alphamax)a))\mua \hspace{1cm}\triangleright \alpha = \alphamax+ (1 -\alphamax)a
\end{align*}
Therefore, $\alpha \in A(\mumix,\muone,\pzeroall)$ but $\alpha = \alphamax + (1-\alphamax)a > \alphamax$, 
which is a contradiction. 

\noindent
\textbf{To prove Statement 1} we use Lemma \ref{lem:iden}.
\begin{align*}
\qzero \subseteq \pzeroall 
&\Rightarrow \mc{F}(\qzero,\muone) \subseteq \mc{F}(\pzeroall,\muone)\\
&\Rightarrow \qzero \cap \mc{F}(\qzero,\muone) \subseteq  \qzero \cap \mc{F}(\pzeroall,\muone) 
= \emptyset.
\end{align*}
Therefore, $\qzero \cap \mc{F}(\qzero,\muone)=\emptyset$ and so by Lemma \ref{lem:iden}, $\mc{F}(\qzero,\muone)$ is identifiable.
 
To prove that $\mc{F}(\qzero,\muone)$ and $\mc{F}(\pzeroall,\muone)$ contain the same set of mixtures, we need to show that $\mc{F}(\qzero,\muone) \subseteq \mc{F}(\pzeroall,\muone)$ and $\mc{F}(\pzeroall,\muone) \subseteq \mc{F}(\qzero,\muone)$, where here we use $\subseteq$ to mean contains a subset of the same unique probability measures. It is already clear that $\mc{F}(\qzero,\muone) \subseteq \mc{F}(\pzeroall,\muone)$ and $\mc{F}(\pzeroall,\muone) \subseteq \mc{F}(\qzero,\muone)$ follows from Statement 2, because any $\mumix \in \mc{F}(\pzeroall,\muone)$ can be represented as a mixture of $\muone$ and some $\mumax \in \qzero$ with $\alphamax$ giving the mixing proportion.

\noindent
\textbf{To prove Statement 3} 
From statement 2, for some $\mumax \in \qzero$,
\begin{align*}
&\alpha\muone+ (1-\alpha) \muzero = \alphamax \muone + (1-\alphamax)\mumax\\
& \Rightarrow \muzero = \nicefrac{(\alphamax-\alpha)}{(1-\alpha)}\muone + \paren*{1-\nicefrac{(\alphamax-\alpha)}{(1-\alpha)}}\mumax.
\end{align*}
Because $\nicefrac{(\alphamax-\alpha)}{(1-\alpha)}\in (0,1)$, $\muzero \in \mc{F}(\pzeroall,\muone)$. Let $a =  \max A(\muzero,\muone, \pzeroall)$ the maximum
proportion of $\muone$ in $\muzero$, with corresponding $\mua \in \pzeroall$ 
such that $\muzero = a \muone + (1-a) \mua$.  
Then, 
\begin{align*}
\mumix &=\alpha \muone + (1-\alpha) \muzero\\
&=\alpha \muone + (1-\alpha) (a \muone+(1-a)\mua)\\ 
&=(\alpha + a -\alpha a) \muone + (1-\alpha -a + \alpha a)\mua
.
\end{align*}
Because the choice of $a$ ensures the maximum proportion on $\muone$
for representing $\muzero$, 
$\mua$ cannot be expressed as a mixture containing $\muone$.
Therefore, 
\begin{align*}
\alphamax = \alpha + a -\alpha a = \alpha + (1-\alpha) a
 \end{align*}
 giving
\begin{align*}
 \alphamax - \alpha= (1-\alpha) \max A(\muzero,\muone, \pzeroall) < \max A(\muzero,\muone, \pzeroall) 
 .
\end{align*}
\end{proof2}

\section{Algorithm development}
\label{sec:algo}

In this section, we formulate the estimation of the mixing proportion in terms of a new set of parameters, $\betavec$. 
We then develop an efficient algorithm by taking advantage of the special optimization surface over $\alpha$, elucidated by Theorem \ref{thm:famQ0} and depicted in Figure \ref{fig:alpha}.

Let $X_1$ be an i.i.d.~sample from the first component and $X$ an i.i.d.~sample from the mixture. To approximate the mixture, a common approach is to use 
\begin{align}
\label{eq:kCompMix}
\hat{\mixf}(x) = \sum_{i=1}^\nkernels \mixwt_i \kappa_i(x) \hspace{1cm} \text{$\mixwt_i\in (0,1)$, $\sum_{i=1}^\nkernels \mixwt_i=1$,}
\end{align}
where $\kappa_i$'s are probability density functions (pdfs) or probability mass functions (pmfs). For continuous random variables, for example, a typical setting is a multivariate Gaussian, $\kappa_i(x) \propto \exp(- \|x - x_i\|_2^2/\sigma^2)$ with learned or predefined centers $x_i$ as in mixture models or examples as centers as in kernel density estimation. For discrete random variables, the $\kappa_i$'s could be multinomials.

To relate $f_1$ to $\mixf$, we similarly approximate $f_1$ by re-weighting the kernels. To do so, we introduce a vector-valued variable $\bm{\beta}=(\beta_1,\ldots, \beta_k)$, where $\beta_i \in (0,1]$, and define
\begin{align*}
\newcone(x|\bm{\beta}) = \frac{\sum_{i=1}^k \beta_i\mixwt_i\kappa_i(x)}{\sum_{i=1}^k \beta_i\mixwt_i} 
\end{align*}
and
\begin{align*}
\newczero(x|\bm{\beta}) &= \frac{\hat{\mixf}(x)-(\sum_{i=1}^k \beta_i\mixwt_i)\newcone(x|\bm{\beta})}{1-\sum_{i=1}^k \beta_i\mixwt_i}
= \frac{\sum_{i=1}^k (1-\beta_i)\mixwt_i\kappa_i(x)}{\sum_{i=1}^k (1-\beta_i)\mixwt_i} 
\end{align*}
giving
\begin{align}
\hat{\mixf}(x)= \left(\sum_{i=1}^k \beta_i\mixwt_i\right)\newcone(x|\bm{\beta}) + \left(1-\sum_{i=1}^k \beta_i\mixwt_i\right)\newczero(x|\bm{\beta})\label{eq:hhat}
.
\end{align}
Notice that $\sum_{i=1}^k \beta_i\mixwt_i \leq 1$ and, thus, $\newcone(x|\betavec)$ and $\newczero(x|\betavec)$ are pmfs for discrete $x$ and pdfs for continuous $x$. It follows that $\hat{\mixf}$ is a mixture with components $\newcone(x|\betavec)$ and $\newczero(x|\betavec)$, where $\alpha = \sum_{i=1}^k \beta_i\mixwt_i$. Intuitively, the $\beta_i$'s are larger if $\cone$ is more similar to $\mixf$; correspondingly, the proportion $\sum_{i=1}^k \beta_i\mixwt_i$ should be larger, because $\cone$ accounts for more of $\mixf$.


Our goal is to obtain an estimate for $\betavec$ with maximal $\sum_{i=1}^\nkernels \beta_i  \mixwt_i$ that minimizes the KL-divergence to the true distributions, $\textrm{KL}(f_1 || h_1(\cdot|\betavec))$ and $\textrm{KL}(\mixf || \hat{\mixf})$; i.e., maximizes the likelihood of $h_1(x|\betavec)$ under sample $X_1$ and the likelihood of $\hat{\mixf}(x)$ under sample $X$. Note that $\hat{\mixf}(x)$ cannot be used in the likelihood function directly because it is not really a function of $\betavec$, as they algebraically cancel. Thus, we define $h(x|\bm\beta)$, similar to $\hat{\mixf}(x)$, obtained by replacing $h_1(x|\bm{\beta})$ in Equation \ref{eq:hhat} by $\hat{f}_1(x)$, an estimate of $f_1$ obtained from $X_1$: 
\begin{align}
h(x|\bm{\beta})= \left(\sum_{i=1}^k \beta_i\mixwt_i\right)\hat{f}_1(x) + \left(1-\sum_{i=1}^k \beta_i\mixwt_i\right)\newczero(x|\bm{\beta})
.
\end{align}

\noindent The combined log-likelihood of $\betavec$ under these models is
\begin{align}
\label{eq:LLGamma}
\mc{L}(\betavec|X,X_1) &=\gamma \mc{L}(\betavec|X)+ \gamma_{1} \mc{L}(\betavec|X_1)
\end{align}
\noindent where
\begin{align*}
 \mc{L}(\betavec|X_1) &=  \sum_{x \in X_1} \log \newcone(x|\betavec) \\
\mc{L}(\betavec|X) &= \sum_{x \in X} \log  h(x|\bm{\beta})
\end{align*}
\noindent and $\gamma$ and $\gamma_{1}$ are nonnegative coefficients. We will later explore two scenarios: (i) $\gamma = \gamma_{1} = 1$ that equally weights each example, and (ii) $\gamma = \nicefrac{1}{|X|}$, $\gamma_{1} = \nicefrac{1}{|X_1|}$ that equally weights each sample. Since the weights do not influence the remainder of this section, we will simply assume that $\gamma = \gamma_{1} = 1$.

There are two remaining issues for the problem specification: concavity and enforcing the \canon. First, although $\mc{L}(\bm{\beta}|X)$ is concave, $ \mc{L}(\bm{\beta}|X_1)$ might not be concave, because 
\begin{align*}
\log \newcone(x|\bm{\beta}) = \log \sum_{i=1}^\nkernels \beta_i \mixwt_i \kappa_i(x) - \log \sum_{i=1}^\nkernels \beta_i \mixwt_i
\end{align*}
has a concave first component and convex second component. Interestingly, however, as described below, by enforcing the max-canonical form in our algorithm, we will be able to avoid this issue.

\begin{algorithm}[H]
\caption{The \AlgName\ algorithm for class prior estimation.} \label{alg_main}
\begin{algorithmic}[] 
\REQUIRE sample $X, X_1$
\ENSURE $\alphamax$
\STATE // Solve level-set optimization for the following set of $\alpha$; for example,
\STATE $c \leftarrow [0.01, 0.02, \ldots, 0.98, 0.99]$
\STATE $n_\alpha \gets $ length($c$)
\FOR{$j = 1, \ldots, n_\alpha$}
\STATE $\ell(j) \gets \max_{\sum_{i=1}^\nkernels \beta_i \mixwt_i = c(j)} \mc{L}(\betavec|X_1,X)$  
\ENDFOR
\STATE // Smooth $\ell$ using median of $2k$-nearest neighbors; typically, $k=3$
\STATE $\ell_{\text{smooth}} \gets \ell$
\FOR{$j = k+1, \ldots, (n_\alpha - k)$}
\STATE $\ell_{\text{smooth}}(j) \gets \text{median}(\ell(j-k), \ldots, \ell(j+k))$
\ENDFOR
\STATE $\ell \gets \ell_{\text{smooth}}$
\STATE // Scale $\ell$ between $0$ and $1$
\STATE $\ell \gets \nicefrac{(\ell - \text{min}(\ell))}{(\text{max}(\ell) -\text{min}(\ell))}$
\STATE // Compute the difference between slopes before and after $j$ using window $win$.
\STATE $\Delta\text{slope} \gets 0$
\FOR{$j = win+1, \ldots, n_\alpha-win$}
\STATE slope-before$(j) \gets$ slope of the linear fit to $\braces*{c(j),\ell(j)}_{j-win}^{j}$. 
\STATE slope-after$(j) \gets$ slope of the linear fit to $\braces*{c(j),\ell(j)}_{j}^{j+win}$.
\STATE $\Delta\text{slope}(j) \gets$ slope-before$(j)-$slope-after$(j)$
\ENDFOR
\STATE // Divide by $1-\ell$ plus a small positive constant $\epsilon$.
\STATE heuristic$ \gets \nicefrac{\Delta\text{slope}}{(1-\ell + \epsilon)}$.
\STATE $\alphamax \leftarrow  c(\text{index-of-max}(\text{heuristic})).$
\end{algorithmic}
\end{algorithm}

We propose a novel algorithm to find $\betavec$ such that $\sum_i \beta_i \mixwt_i = \alphamax$, using two steps. First, we estimate the log-likelihood of $\alpha$ at several points $c_j \in (0,1)$ by using $$\max_{\sum_{i=1}^\nkernels \beta_i \mixwt_i = c_j} \mc{L}(\betavec|X_1,X).$$ By optimizing over level sets, we generate an optimization surface in terms of $\alpha$, as suggested in Figure \ref{fig:alpha}. According to Theorem 1, the surface should have an initial flat region, until $\alphamax$; then the likelihood should begin to deteriorate. The second step, therefore, is to identify this point $\alphamax$ on this surface. For the first step, the constraint makes the convex part of the objective become a constant, resulting in a concave maximization 
\begin{align*}
\argmax_{\sum_{i=1}^\nkernels \beta_i \mixwt_i = c_j} \mc{L}(\betavec|X_1,X)
&=
\argmax_{\sum_{i=1}^\nkernels \beta_i \mixwt_i = c_j}  \sum_{x \in X_1} \log \newcone(x|\betavec) + \sum_{x \in X} \log  h(x|\bm{\beta})\\
&=
\argmax_{\sum_{i=1}^\nkernels \beta_i \mixwt_i = c_j}  \sum_{x \in X_1} \log \sum_{i=1}^\nkernels \beta_i \mixwt_i \kappa_i(x) - \log c_j + \sum_{x \in X} \log  h(x|\bm{\beta})\\
&=
\argmax_{\sum_{i=1}^\nkernels \beta_i \mixwt_i = c_j}  \sum_{x \in X_1} \log \sum_{i=1}^\nkernels \beta_i \mixwt_i \kappa_i(x) + \sum_{x \in X} \log  h(x|\bm{\beta})
.
\end{align*}
For the second step, we identify the point where the slope changes the most. To improve robustness of this step, we smoothed the curve using median values of closest $k$ neighbors. The full algorithm, referred to here as \AlgName, is summarized in Algorithm \ref{alg_main}. 

\section{Transformations that preserve $\alphamax$}

In this section, we discuss one approach to practically learning on (high-dimensional) multivariate data. The approach consists of transforming the multivariate data to univariate data in such a way that $\alphamax$ for the transformed data is unchanged. At first glance, this transformation may seem unnecessary, as our class prior estimation algorithm is a generic non-parametric approach.\footnote{Radial basis function networks are universal approximators \citep{Park1991} and kernel density estimators are consistent for any density \citep{Scott1979}.} For example, to extend to structured data, kernel components could be chosen to measure similarities for these objects. Despite this generality, practical kernel density estimation under high-dimensional and/or structured data can be problematic. There are curse-of-dimensionality issues with high-dimensional density estimation, both in theory \citep{liu2007sparse} and in practice \citep{scott2008curse}. One strategy for high-dimensional density estimation is to use product kernels \citep{cooley1998classification, liu2007sparse}. We propose to instead transform the multivariate data to univariate data using a probabilistic classifier, taking advantage of the fact that we are in a classification setting. Importantly, we can prove that this transformation preserves $\alphamax$ in \autoref{thm:classifier}; that is, the class prior for this univariate data is equal to the class prior for the original multivariate data. Once the data is transformed to univariate, we avoid the curse-of-dimensionality for density estimation. Although we still have to deal with high-dimensional spaces in classification, we can exploit a richer set of techniques to overcome these problems \citep{Hastie2001}.

The procedure to obtain this univariate transformation reduces to a classification problem. We construct a training data set with all the mixture examples labeled as class $0$ and component examples as class $1$. The classifier trained on this data set approximates the probability that $x$ is labeled. These probability values comprise the transformed univariate data set.
 
 \begin{figure}[]
  \centering
  \includegraphics[scale=1.0]{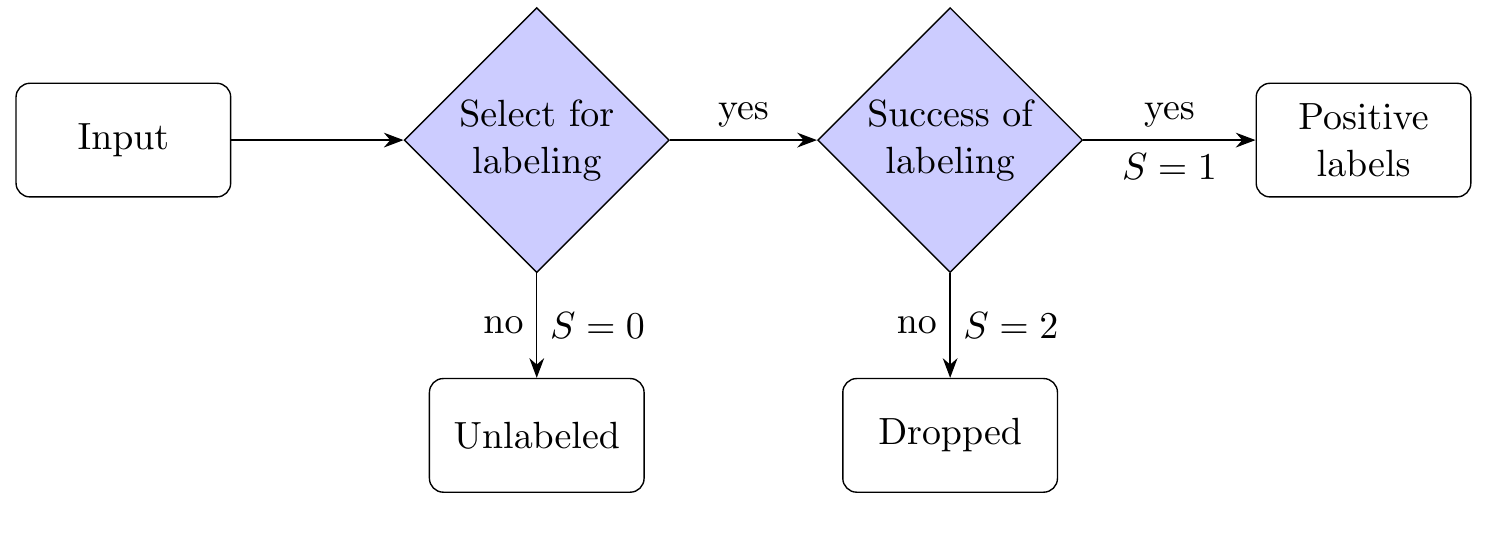}
\caption{The implicit labeling procedure, with $S \in \{0,1,2\}$. In the first step, the sample is randomly selected to attempt labeling, with some probability independent of $\alpha$. If it is not selected, it is added to the ``Unlabeled" set. If it is selected, then labeling is attempted. If the true label is $Y = 0$, then the attempted labeling will fail, and $x$ is added to the ``Dropped" set. If the true label is $Y=1$, then with some probability (again independent of $\alpha$) the labeling will succeed, and $x$ will be added to ``Positive labels"; otherwise, labeling fails and $x$ is added to ``Dropped".}\label{fig_labeling}
\end{figure}

To show that this approach preserves $\alphamax$, we outline a set of probabilistic assumptions. Let $X$ be the random variable distributed according to the mixture $\mixf$ and $Y$ be the unobserved random variable giving the label of the component from which $X$ was generated.\footnote{Note the abuse of notation; previously, we used $X$ to indicate the mixture sample but in this Section we temporarily override this notation.} Let $S \in \{0,1,2\}$ be a selection random variable, whose value determines the sample to which $X$ is added. When $S=1$, $X$ is added to the component sample; when $S=0$, $X$ is added to the mixture sample; and when $S=2$, $X$ is not added to either of the samples. This procedure is demonstrated in Figure \ref{fig_labeling}. 
 
The following assumptions on $S$ are required so that the mixture sample and the component sample have the correct distribution 
\begin{equation}
\begin{aligned}
\label{eq:SXY_dependence}
p(X=x|S=0)&=p(X=x) \\
p(X=x|S=1)&=p(X=x|Y=1),
\end{aligned}
\end{equation}
where $p(X=x) = f(x)$ and $p(X=x|Y=1) = \cone(x)$,
but we use the more general notation to incorporate random variable $S$.
Moreover $f$, being a non-trivial mixture, imposes the following constraints on $X$ and $Y$.
\begin{equation}
\begin{aligned}
\label{eq:XY_nontrivial}
p(Y=1) &\neq 0\\
p(X=x|Y=1)&\neq p(X=x).
\end{aligned}
\end{equation}
We emphasize that it is impossible to satisfy conditions from \autoref{eq:SXY_dependence} and \autoref{eq:XY_nontrivial} when $S=2$ is independent of $X$, as shown by the following corollary. In other words, the probability that a point in the mixture's support remains unsampled varies with the point in question. Moreover, as shown by the corollary, it also implies that $p(S=2)\neq 0$; i.e., $S$ cannot be defined in a way that allows all points from the mixture population to be sampled. This is a subtle but important result as the the assumptions from \autoref{eq:SXY_dependence} have been made previously \citep{Phillips2009}, with a random variable that only takes two values $\{0,1\}$. Unfortunately, this leads to inconsistent assumptions.
\begin{proposition}
For random variables $X$, $Y$ and $S$ defined above, if the conditions from \autoref{eq:SXY_dependence} are satisfied, then $S=2$ is dependent on $X$. 
Moreover, 
$p(S=2)\neq 0$.
\end{proposition}
\begin{proof}

\begin{align*}
p(S=2|X=x) &=1- p(S=0|X=x) - p(S=1|X=x)\\
&= 1- p(S=0) - \frac{p(X=x|S=1)}{p(X=x)}p(S=1) \tag{because $S=0$ and $X$ are independent from \autoref{eq:SXY_dependence}}\\
&= 1- p(S=0) - \frac{p(X=x|Y=1)}{p(X=x)}p(S=1)    \tag{from \autoref{eq:SXY_dependence}}               
\end{align*}

\noindent The probability $p(S=2 | X = x)$ is independent of $x$ only if $\nicefrac{p(X=x|Y=1)}{p(X=x)}$ is a constant with respect to $x$. However, 
such a constant can only be $1$, which is inconsistent with \autoref{eq:XY_nontrivial}. To see why this constant can only be $1$, assume
$\nicefrac{p(X=x|Y=1)}{p(X=x)} = c$. Integrating over $x$ on both sides gives
$\int_{\mathcal{X}} p(X=x|Y=1) dx = c \int_{\mathcal{X}} p(X=x) dx$. Since
both integrals are $1$, it follows that $c = 1$. Therefore, $S=2$ is not independent of $X$. 


To prove $p(S=2)\neq0$ we give a proof by contradiction: $p(S=2)=0$ implies that  $p(S=2|X=x)=0$, which further implies $S=2$ is independent of $X$. However, this is not possible as shown above, hence the contradiction. 
\end{proof}

We now define a transformation that preserves the mixing proportion, ensuring that the $\alphamax$ for the transformed (univariate) data is equal to the $\alphamax$ for the original (multivariate) data.
In the following theorem, we prove that the probabilistic classifier 
\begin{equation}
\tau(x) = p(S=1|X=x,S\in\{0,1\}) \label{eq_tau}
\end{equation}
preserves $\alphamax$. This theorem references a later more general theorem for other univariate transforms. We focus on this transformation because it is a concrete, useful example. Moreover, once $\alphamax$ (i.e., $p(Y=1)$) is estimated, $\tau$ can be used to obtain a traditional classifier because it gives an estimate of $p(Y=1|X=x)$, as stated in the theorem. For this theorem, we assume that $X$ is a continuous random variable with a pdf, for simplicity of presentation. This assumption can be generalized to discrete or mixed random variables, with slightly more cumbersome notation. 
\begin{theorem}[\textbf{$\alphamax$-preserving transformation}]
\label{thm:classifier}
Let $X,Y,S$ be the random variables defined above under the dependence assumptions in \autoref{eq:SXY_dependence}.
Assume $X$ is a continuous random variable with pdf $\mixf$
and let $X_1$ be the random variable corresponding to the first component in the mixture with pdf $\cone$. Let $\mu$ and $\muone$ be the corresponding measures. 
Let $\tau(x) = p(S=1|X=x,S\in\{0,1\})$.
Then the random variables $\tau(X)$ and $\tau(X_1)$, with measures $\numix$ and $\nuone$,
have the same mixing proportion
\begin{align*}
\alphamax = \inf R(\mumix, \muone) = \inf R(\numix, \nuone)
.
\end{align*}
and so $\tau$ is an $\alphamax$-preserving transformation.

Moreover, $\tau$ can then also be used for classification because
\begin{align*}
p(Y=1|X) = \frac{cp(Y=1)\tau(X)}{1-\tau(X)}
\end{align*}
where $c=\nicefrac{p(S=0)}{p(S=1)}$ and 
$\tau(x) = \cone(x) (\cone(x) + c \mixf(x))^{-1} < 1$ for all $x \in \xset$.
\end{theorem}
\begin{proof}
Let 
\begin{align*}
\taudiv(x) = \left\{ \begin{array}{ll}
         \mixf(x) / \cone(x) & \mbox{if $\cone(x) > 0$}\\
        \infty & \mbox{if $\cone(x) = 0$}.\end{array} \right.
\end{align*}
To prove that $\tau$ is $\alphamax$-preserving, we simply need to prove that
$\tau$ satisfies the conditions of \autoref{thm:transformpdf} (proved below); i.e.,
that $\tau$ is the composition of a one-to-one function and $\taudiv$. 
\begin{align}
\tau(x)&= \frac{p(S=1,X=x,S\in\{0,1\})}{p(X=x,S\in\{0,1\})} \nonumber\\
&=\frac{p(S=1,X=x)}{p(X=x,S=0)+p(X=x, S=1)}\notag\\
&=\frac{p(X=x|S=1)p(S=1)}{p(X=x|S=0)p(S=0)+p(X=x|S=1)p(S=1)}\nonumber\\
&=\frac{p(X=x|Y=1)p(S=1)}{p(X=x)p(S=0)+p(X=x|Y=1)p(S=1)} \tag{applying \autoref{eq:SXY_dependence}}\nonumber\\
&=\frac{1}{1+ \frac{p(S=0)}{p(S=1)}\frac{p(X=x)}{p(X=x|Y=1)}} \nonumber \\
&=\frac{1}{1+ c\frac{p(X=x)}{p(X=x|Y=1)}} \label{eq_eqtaudiv}\\
&=\frac{1}{1+ c\taudiv(x)} \nonumber
\end{align}
Therefore, $\tau$ satisfies the conditions of \autoref{thm:transformpdf}, because it is a composition $\tau = H \circ \taudiv$,
with one-to-one function $H(z) = (1 + c z)^{-1}$. 
Rearranging terms in \eqref{eq_eqtaudiv}, we can further see that
\begin{align*}
\tau(x)
&=\frac{1}{1+ c\frac{\mixf(x)}{\cone(x)}} = \cone(x) (\cone(x) + c \mixf(x))^{-1} 
.
\end{align*}

\noindent For the second claim, again starting at  \eqref{eq_eqtaudiv}, we get the following by using Bayes rule
\begin{align}
\tau(x)
&=\frac{1}{1+ c\frac{p(X=x)}{p(X=x|Y=1)}} \nonumber\\
&=\frac{1}{1+ c\frac{p(Y=1)}{p(Y=1|X=x)}} \label{eq:SX2}
.
\end{align}
Rearranging \autoref{eq:SX2},
\begin{align*}
p(Y=1|X=x) = \frac{cp(Y=1)\tau(x)}{1-\tau(x)}
.
\end{align*}
\end{proof}
%

The constant $c$ can be estimated by dividing the size of the mixture sample by the size of the component sample.

\subsection{General theorem for univariate transforms}
In this section, we prove a more general theorem about univariate transforms, with conditions on the transform which we above showed that $\tau$ satisfies. First, we will prove that any transform $\tau$ gives an $\alphamaxtau$ in the transformed space that is an upper bound on the $\alphamax$ in the original space. Then we will show the conditions on $\tau$ that ensure the $\alphamax$ are equal in the two spaces.

\begin{lemma}
\label{lem:alphamaxleqpdf}
Let $X$ and $X_1$ be random variables with 
measures $\mumix$ and $\muone$ respectively,
satisfying $\mumix \in \mc{F}(\pzeroall,\muone)$.
Let $\tau$ be any function defined on $\xset$, and $\numix$, $\nuone$
the measures for the random variables $\tau(X)$,
$\tau(X_1)$ respectively for $\sigma$-algebra $\Salgebra_\tau$.  Let 
\begin{align*}
\alphamax &= \inf \mc{R}(\mumix,\muone) \\
\alphamaxtau &=\inf \mc{R}(\numix,\nuone)
.
\end{align*}
Then $\alphamaxtau \geq \alphamax$.
\end{lemma}
\begin{proof}
First, 
there is a corresponding random variable $X_0$ with measure $\muzero$ such that $\mumix = \alphamax \muone + (1-\alphamax) \muzero$. 

For $Z \sim Bernoulli(\alphamax)$ (independent of $X, X_0$ and $X_1$), it is easy to show that $X = ZX_1 + (1-Z) X_0$. 
Therefore,
\begin{align*}
\tau(X) &= \tau \paren*{ZX_1 + (1-Z) X_0}\\
&= \begin{cases}
            \tau(X_1) & \ when\  Z=1\\  
            \tau(X_0) & \ when\  Z=0
            \end{cases}\\
&= Z\tau(X_1) + (1-Z) \tau(X_0).            
\end{align*}
Thus $\tau(X)$ is a mixture containing $\tau(X_1), \tau(X_0)$ with $\alphamax$ as the mixing proportion. 
In other words $\numix=\alphamax\nuone + (1-\alphamax)\nuzero$. 
For any $A \in \Salgebra_\tau$ such that $\nuone(A) > 0$, since $\nuzero(A) \ge 0$,
we get $\numix(A) / \nuone(A) \ge \alphamax$.
Since this is true for any $\nuzero(A) \ge 0$,
we get $\alphamaxtau = \inf R(\numix, \nuone) \ge \alphamax$.
\end{proof}
Now to get equality, we propose the following transformation: $\tau = H \circ \taudiv$,
where $H$ is a one-to-one function and 
\begin{align*}
\taudiv(x) = \left\{ \begin{array}{ll}
         \mixf(x) / \cone(x) & \mbox{if $\cone(x) > 0$}\\
        \infty & \mbox{if $\cone(x) = 0$}.\end{array} \right.
\end{align*}

\begin{theorem}
\label{thm:transformpdf}
Let $X$ and $X_1$ be random variables with 
pdfs $\mixf$ and $\cone$ and
measures $\mumix$ and $\muone$ respectively,
satisfying $\mumix \in \mc{F}(\pzeroall,\muone)$.
For $\overline{\mb{R}}^+=\mb{R}^+ \cup \{0,\infty\}$ and an abstract space $\xsett$,
given any one-to-one function $H: \overline{\mb{R}}^+ \rightarrow \xsett$,
define function $\tau: \xset \rightarrow \xsett$
 \begin{align*}
 \tau = H \circ \taudiv
 .
 \end{align*}
 Let $\numix$ and $\nuone$ be 
the measures for the random variables $\tau(X)$,
$\tau(X_1)$ respectively for $\sigma$-algebra $\Salgebra_\tau$ on $\xsett$.  Let 
\begin{align*}
\alphamax &= \inf \mc{R}(\mumix,\muone) \\
\alphamaxtau &=\inf \mc{R}(\numix,\nuone)
.
\end{align*}
Then $\alphamaxtau = \alphamax$.
\end{theorem}
\begin{proof}
Lemma \ref{lem:alphamaxleqpdf} already proves that $\alphamaxtau \ge \alphamax$.
Therefore, we simply need to prove that $\alphamaxtau \le \alphamax$.

\textbf{Part 1:} First, we define a set $\Bepsilon \in \Salgebra$ such
that $\mumix(\Bepsilon) \le  (\alphamax + \epsilon) \muone(\Bepsilon)$ for any $\epsilon>0$.
Let $T \subseteq \overline{\mb{R}}^+$ be the range of $\taudiv$. Let $\Tepsilon = T \cap [\alphamax, \alphamax +\epsilon)$ for some $\epsilon > 0$. Notice first that $\Tepsilon \neq \emptyset$ because
there exist $x \in \xset$ with $\cone(x)>0$ such that $\taudiv(x) = \mixf(x)/\cone(x) \le \alphamax + \epsilon$ (using $\alphamax=\inf \braces*{\nicefrac{\mixf(x)}{\cone(x)}: x \in \xset, \ \cone(x)>0}$ from Lemma \ref{lem:Ah}). 
Let $\Bepsilon$ be the inverse image of $\Tepsilon$ under $\taudiv$
\begin{align*}
\Bepsilon = \braces*{x \in \mb{R}^d: \taudiv(x) \le \alphamax + \epsilon, \cone(x)>0}
.
\end{align*}
 Thus for all $x \in \Bepsilon$, $\mixf(x) \le (\alphamax + \epsilon) \cone(x)$. Integrating over $\Bepsilon$, 
\begin{align}
\mumix(\Bepsilon) = \int_{\Bepsilon} \mixf(x) dx \le (\alphamax + \epsilon) \int_{\Bepsilon} \cone(x) dx = (\alphamax + \epsilon) \muone(\Bepsilon) \label{eq:BepInq}.
\end{align}

\textbf{Part 2:} Now we show that $\muone(\Bepsilon) > 0$.
Because $\alphamax = \inf R(\mumix,\muone)$, for a given $\epsilon > 0$, there exists $\Aepsilon$ such that $\muone(\Aepsilon)>0$ and  $\nicefrac{\mumix(\Aepsilon)}{\muone(\Aepsilon)} < \alphamax+\epsilon$. Let $\Aepsilon^0 = \Bepsilon \cap \Aepsilon$.
%
 
For all $x \in \Aepsilon^0$, $\mixf(x) < (\alphamax + \epsilon) \cone(x)$. Integrating over $\Aepsilon^0$ on both the sides, we get  $\int_{\Aepsilon^0} f(x) dx < (\alphamax + \epsilon) \int_{\Aepsilon^0} f_1(x) dx$. Thus $\mumix\paren*{\Aepsilon^0} < (\alphamax + \epsilon) \muone\paren*{\Aepsilon^0}$. Similarly, for all $x \in \Aepsilon\setminus\Aepsilon^0$, $\mixf(x) \geq (\alphamax + \epsilon) \cone(x)$ and consequently, $\mumix\paren*{\Aepsilon\setminus \Aepsilon^0} \geq (\alphamax + \epsilon) \muone\paren*{\Aepsilon\setminus \Aepsilon^0}$.
 
 Suppose $\muone\paren*{\Aepsilon^0}=0$. It follows that $\mu\paren*{\Aepsilon^0}=0$ and consequently,
\begin{align*}
\mu(\Aepsilon) &= \mu\paren*{\Aepsilon^0} + \mu\paren*{\Aepsilon \setminus \Aepsilon^0}\\
       &= \mu\paren*{A \setminus A^0}\\
       &\geq (\alphamax + \epsilon) \muone\paren*{\Aepsilon \setminus \Aepsilon^0}\\
       &= \delta \paren*{\muone \paren*{\Aepsilon \setminus \Aepsilon^0} + \muone \paren*{\Aepsilon^0}} \tag{because $\muone\paren*{\Aepsilon^0}=0$}\\
       &=\delta \muone(\Aepsilon)
\end{align*}
This contradicts the given statement. Hence $\muone\paren*{\Aepsilon^0}>0$.
Because $\Aepsilon^0 \subseteq \Bepsilon$, it follows that $\muone\paren*{\Bepsilon}>0$ as well.  

\textbf{Part 3:} Now we show that $\inf R(\lammix,\lamone) \leq \alphamax$, where $\lammix,\lamone$ are probability measures induced by $\mu$, $\muone$, respectively, under $\taudiv$. 
Because $\Bepsilon$ is the inverse image of $\Tepsilon$ under $\taudiv$, $\lammix(\Tepsilon)=\mumix\paren*{\Bepsilon}$ and $\lamone(\Tepsilon)=\muone\paren*{\Bepsilon}$. Hence $\lammix(\Tepsilon)\le(\alphamax +\epsilon)\lamone(\Tepsilon)$ and $\lamone(\Tepsilon)>0$. Now, because $(\alphamax +\epsilon)\ge\nicefrac{\lammix(\Tepsilon) }{\lamone(\Tepsilon)} \in R(\lammix,\lamone$), $\inf R(\lammix,\lamone) \le \alphamax + \epsilon$. This is true for all $\epsilon > 0$. Thus $\inf R(\lammix,\lamone) \leq \alphamax$. 

\textbf{Part 4:}
Finally, we show that $\alphamaxtau \leq \alphamax$. Because $\numix, \nuone$ are probability measures induced by $\mu,\muone$, respectively, under the transformation $\tau = H \circ \taudiv$, $\lammix, \lamone$ also induce $\numix, \nuone$, respectively, under $H$. Let $H(\Tepsilon)=\Tepsiloni$. Because $H$ is one-to-one, $\Tepsilon$ is the inverse image of $\Tepsiloni$ and $\numix(\Tepsiloni)=\lammix(\Tepsilon),\nuone(\Tepsiloni)=\lamone(\Tepsilon)$. Now, because $\nuone(\Tepsiloni)=\lamone(\Tepsilon)>0$  and $(\alphamax +\epsilon)\ge\nicefrac{\lammix(\Tepsilon)}{\lamone(\Tepsilon)}= \nicefrac{\numix(\Tepsiloni) }{\nuone(\Tepsiloni)} \in R(\numix,\nuone)$, $\alphamaxtau = \inf R(\numix,\nuone) \le \alphamax + \epsilon$. This is true for all $\epsilon > 0$. Thus $\alphamaxtau \leq \alphamax$.  
\end{proof}

\section{Related work}

The problem of class prior estimation appears in a variety of forms and learning contexts. Early classification approaches generally operated under the umbrella of sample selection bias theory \citep{Heckman1979, Cortes2008}, where class priors but not the class-conditional distributions differ among labeled and unlabeled data. These methods assume the existence of both positives and negatives in the labeled set and estimate class priors using various forms of iterative learning \citep{Latinne2001, Vucetic2001, Saerens2002}. Interestingly, the expectation-maximization (EM) approach by \citet{Latinne2001} and \citet{Saerens2002} can be reformulated as minimization of Kullback-Leibler distance between labeled and unlabeled data \citep{duPlessis2012}, resulting in a convex objective. This formulation further allows distribution matching to be generalized to other distance functions, such as the Pearson divergence \citep{Pearson1900}.

The positive-unlabeled scenario was specifically considered by \citet{Elkan2008} and \citet{Phillips2009} who investigated the relationship between traditional classifiers (between positive and negative data) and non-traditional classifiers (between labeled and unlabeled data). Assuming that a non-traditional classifier can learn the posterior probability that a data point is labeled, \citet{Elkan2008} proposed estimators for class priors in the unlabeled data. This approach, however, holds strictly only for class-conditional distributions with disjoint supports; see derivation for $g(x)$ on p.~214 in \citet{Elkan2008}. Their strategy can similarly be reformulated as minimization of the Pearson divergence between labeled data scaled by the unknown class prior and unlabeled data \citep{duPlessis2014b}. This partial distribution matching leads to a compact solution, although it still requires a non-linear fitting step \citep{duPlessis2014b}. Nevertheless, these methods are equivalent in that they minimize the same objective. Another work from this group of methods uses the principles of the EM algorithm to maximize the conditional likelihood of a logistic regression model \citep{Ward2009}. \citet{Ward2009} also provide restrictive conditions that ensure identifiability of class priors and investigate the variance of estimates.

\citet{Scott2009} and \citet{Blanchard2010} provide an extensive theoretical treatment of the subject. Two of their results are particularly relevant for our work: (i) They provide a general non-identifiability result that applies to any probability distribution and also show the existence of $\alphamax$ \citep{Blanchard2010}. However, as mentioned earier, they do not specifically identify $A(\mumix,\muone,\pzeroall)$ as an interval and recognize its relevance for developing practical estimators. (ii) They propose an estimator for $1-\alphamax$ as an infimum over a set of functions. This is a theoretically important result, although to our knowledge this estimator does not lead to an algorithm to compute $\alphamax$ in practice. These results have been recently extended to the cases of classification with asymmetric label noise \citep{Scott2013}. 

Estimation of class priors can also be seen as an instance of parameter learning in two-component mixture models. Here, an extensive and well-studied group of algorithms is available, predominantly based on the EM algorithm \citep{Dempster1977} and its many variants \citep{McLachlan2000}. The identifiability of finite mixtures has been thoroughly studied; e.g., see \cite{Yakowitz1968} and \cite{Tallis1982}. Using samples from both the mixture and component one simplifies the estimation problem; to our knowledge, however, solutions can only be generalized to parametric families. The unsupervised view is attractive because it also ties class prior estimation with hypothesis testing and false discovery rate estimation in statistics \citep{Storey2003, Geurts2011, Ghosal2011}. 

A number of additional supervised approaches have been proposed to address the problems of learning from positive and unlabeled data. Generally, however, most authors are primarily interested in improving accuracy of traditional classification models \citep{Denis1998, Liu2003, Lee2003, Yu2004, Zhang2005}. Various other forms of one-class classification methods and outlier detection can also be used for learning and inference, although the evidence suggests that these strategies are generally inferior \citep{Manevitz2001}. In summary, this variety of approaches suggests deep connections between class prior estimation, hypothesis-testing \citep{Storey2003, Scott2005, Geurts2011, Beana2013}, learning from positive and unlabeled data \citep{Elkan2008}, and cost-sensitive learning \citep{Elkan2001, duPlessis2014}.

\section{Empirical investigation}

In this section, we investigate the practical properties of our approach to estimating the mixing proportion in a controlled, synthetic setting, and subsequently on real-life data. 

\subsection{Experiments on synthetic data}
We explore the impact of the mixing proportion, the separation between the mixing components, and the size of the component sample on the accuracy of estimation. In all experiments, $\alpha$ was varied from $\{0.05, 0.25, 0.50, 0.75, 0.95 \}$, the size of the component sample $X_1$ was varied from $\{100,1000\}$, whereas the size of the mixture sample $X$ was fixed at $10000$. For each set of parameters, $\alpha$ was estimated 50 times from a randomly generated data set. A two-sample t-test was used to estimate the statistical significance that one algorithm was a better estimator than another. A P-value threshold of 0.05 and the Bonferroni correction were used to declare an algorithm a winner over all other algorithms.

\textbf{Data sets: } The univariate data was generated from the mixture of two unit-variance Gaussian distributions and from two unit-scale Laplace distributions, with varying means ($\Delta \mu \in \{1,2,4\}$) in both settings. In the multivariate case, we used the waveform data generator \citep{Breiman1984} adjusted for binary classification ($\xdim = 21$) and have constructed a ten-dimensional sphere of radius one inscribed into a cube with a side of length four, with near-uniformly generated positive (inside the sphere) and negative (outside the sphere) samples ($\xdim = 10$). To apply our model, we have initially trained a binary classifier between positive and unlabeled data and used the distributions of its predictions (through cross-validation) to construct samples $X_1$ and $X$. These univariate samples were subsequently provided to our estimator of mixing proportions.

\subsection{Experiments on real-life data}
We downloaded twelve real-life data sets from the UCI Machine Learning Repository \citep{Lichman2013}. If necessary, categorical features were transformed into numerical using sparse binary representation, the regression problems were transformed into classification based on the mean of the target variable, and the multi-class classification problems were converted into binary by combining original classes. In each data set, a subset of 1000 positive examples (or 100 for smaller data sets) was randomly selected to provide a sample $X_1$ while the remaining data (without class labels) were used as unlabeled data (sample $X$). The true class prior corresponded to the fraction of positives in sample $X$. Note that this experiment differs from the setup by \citet{Elkan2008} in that we use $X$ to to define true class prior, whereas Elkan and Noto use the fraction of positives in $X \cup X_1$. We made an appropriate conversion to obtain comparable results. 

Each experiment was repeated 50 times for a random selection of 1000 positives from the original data set, except for the four smaller data sets where it was set to 100. The maximum data set size was limited at 10000 for the large data sets; in each such case, the positive and negative examples were sampled in a stratified manner. As in the case of synthetic data, a binary classifier between positive and unlabeled data was used to provide univatiate distributions of prediction scores to construct samples $X_1$ and $X$. These univariate samples were subsequently provided to our estimator.

\subsection{Algorithms}
We compared \AlgName\ to two known algorithms for estimating the mixing proportion and two baseline algorithms suggested by Lemma \ref{lem:Ah}. The first algorithm includes Gaussian mixture models (GMM), trained using expectation-maximization (EM). The GMM algorithm was used on $X$, while $X_1$ was used to select the mixing proportion between $\alpha$ and $1 - \alpha$ based on the distance between the inferred means to the centroid of $X_1$. In the case of multivariate data, we first applied the multivariate-to-univariate transformation and then used GMM to infer class priors. We refer to this algorithm as transformed GMM (GMM-T) to distinguish it from the GMM that would be directly applied to multivariate data. The second algorithm includes the \elkan\ method \citep{Elkan2008}. There are two potential variants of the \elkan\ estimator: the main estimator as described in Eq.~4 in \citep{Elkan2008} and the alternative estimator. Each of these estimators can work with three different estimates of the probability $p(S = 1|Y = 1)$, referred to as $e_1$, $e_2$, and $e_3$ in \citep{Elkan2008}, and each of these estimators can use any classifier to learn $p(S = 1|x)$. We have tested a bagged ensemble of $100$ two-layer feed-forward neural networks, each with five hidden units, and a support vector machine with a quadratic kernel and Platt's post-processing \citep{Platt1999}. These models performed similarly well; thus, we only report the results corresponding to the neural network ensembles. In addition, since estimator $e_3$ was significantly inferior to $e_1$ and $e_2$, and since $e_1$ was slightly better than $e_2$ in performance, we only report the results for the $e_1$ estimator. Finally, the alternative estimator had a better performance than the main estimator. Therefore, the results for the \elkan\ algorithm correspond to the alternative estimator using an ensemble of neural networks and the $e_1$ estimate of $p(S = 1|Y = 1)$. The algorithm proposed by \cite{duPlessis2014b} minimizes the same objective as the $e_1$ estimator and, thus, was not used in our experiments. Note that the \elkan\ algorithm may output class priors greater than 1, which can occur when the posterior probability $p(S = 1|x)$ is not accurately learned. 
In the case of \AlgName, we estimated the densities using histograms. The bin-width was chosen to cover the component sample's (after the transformation) range and reveal the shape of its distribution, using the default option in Matlab's \verb+histogram+ function. More bins with the same bin-width were subsequently added to cover the mixture sample's range. We explored two combinations of coefficients $(\gamma, \gamma_{1})$ as specified in \autoref{eq:LLGamma}. The combination $(\nicefrac{1}{|X|}, \nicefrac{1}{|X_1|})$ resulted in slightly better performance than $(1,1)$; thus, we only present the results when the mixture sample and the component sample were equally weighted.

The two baseline algorithms follow from the insights about $\alphamax$ derived in Lemma \ref{lem:Ah}. The first, which we call the \pdfratio\ approach, uses the fact that $\alphamax = \inf R(\mixf, \cone)$, suggesting the approximation $\hat{\alpha} = \min_{x_i \in X_1} \nicefrac{\hat{\mixf}(x_i)}{\hat{\cone}(x_i)}$. The second, which we call the \cdfratio\ approach, uses the fact that $\alphamax$ can be approximated using the cdf $F$ of $\mixf$ and the cdf $F_1$ of $\cone$. To see why, consider the function $\nicefrac{(\mixf-\alpha \cone)}{1-\alpha}$. This function is a pdf, provided the numerator is nonnegative, because it integrates to $1$. The numerator is nonnegative for $\alpha \leq \alphamax$ but not for $\alpha > \alphamax.$ Consequently, $\alphamax$ is the largest value for which $\nicefrac{(F-\alpha F_1)}{(1-\alpha)}$ is a cdf. Finding this $\alphamax$ corresponds to finding the largest $\alpha$ for which $F-\alpha F_1$ is nonnegative and non-decreasing. To execute this search, we obtain estimates of the cdfs, $\widehat{F}_1$ and $\widehat{F}$, and discretize the problem by using only the values of $\widehat{F}-\alpha \widehat{F}_1$ evaluated on the sample from the component. To check for the non-decreasing property, we apply the first finite difference operator on these values to check if the result is nonnegative. The results corresponding to these estimators are provided in Tables \ref{tab:GL_pdf}-\ref{tab:RD_pdf} at the end of this Section. 

To find the optimal $\betavec$ for each level set in Algorithm \ref{alg_main}, we used an interior-point method. The implementation for the minimizer was \verb+fmincon+ in Matlab. The AlphaMax level set optimization is quite simple, with a linear equality constraint and a concave maximization. Consequently, we found that the other optimization approaches we explored, including LBFGS, did not provide any gains. We would like to mention that, at times, we observe uncharacteristically small log-likelihood for the optimal $\betavec$ returned by the optimization routine because of numerical instability; this happens only when mixing proportions approach 0. We correct the log-likelihoods by imposing the constraint that the log-likelihood should be non-increasing with respect to the mixing proportion, which is consistent with the theory.
%

\subsection{Results}
In \autoref{tab:GL} we show mean absolute error from the true mixing proportion over the Gaussian and Laplace data sets and multiple parameter values. We varied the value of the mixing proportion as well as the size of the component sample and compared \AlgName, \elkan, and GMM algorithms. The performance of all three methods was good, with methods generally having more difficulties on poorly separated component distributions, particularly for the small size of the component sample. Since all data sets were univariate, in some cases with well-separated components, we did not expect \AlgName \ to outperform other algorithms. This is particularly the case for the data sets containing mixtures of Gaussian distributions for which the EM algorithm was designed for. Nevertheless, the results provide evidence that \AlgName \ performs well on data sets with low separation between mixing components.

In \autoref{fig:GL}, we visualize the variance of estimates by providing box plots over a set of data sets and true mixing proportions. \autoref{fig:GL} also shows the log-likelihood plots on a random selection of one of the 50 data sets. As mentioned earlier, the log-likelihood functions were used by our automated procedure to identify the inflection point; i.e., a point at the end of the initial flat region. Although the development of an automated procedure is important, we observe that these log-likelihood plots readily provide a useful tool for practitioners to visually select the mixing proportion and gain insight into the reliability of the estimate. We anticipate that these plots might be preferred in settings where a single data set is considered.
\renewcommand{\multirowsetup}{\centering}
\newlength{\LL} \settowidth{\LL}{Data}
\setlength{\tabcolsep}{3pt}


\begin{table}[t]
\centering
\caption{Mean absolute difference between estimated and true mixing proportion over a selection of true mixing proportions and the following data sets: $\mathcal{N}$ = Gaussian with $\Delta \mu \in \{1, 2, 4\}$, $\mathcal{L}$ = Laplace with $\Delta \mu \in \{1, 2, 4\}$. Statistical significance was evaluated by comparing the AlphaMax method, the Elkan-Noto algorithm, and the Gaussian Mixture Model (GMM). The bold font type indicates the winner and the asterisk indicates statistical significance.}
\footnotesize
\tracingtabularx
\begin{tabularx}{0.70\linewidth}
{|>{\setlength{\hsize}{0.19000\hsize}}X|
>{\setlength{\hsize}{0.10000\hsize}}X|
>{\setlength{\hsize}{0.121667\hsize}}X|
>{\setlength{\hsize}{0.121667\hsize}}X|
>{\setlength{\hsize}{0.121667\hsize}}X|
>{\setlength{\hsize}{0.121667\hsize}}X|
>{\setlength{\hsize}{0.121667\hsize}}X|
>{\setlength{\hsize}{0.121667\hsize}}X|
}\hline 
\multirow{2}{1\LL}{Data} &
& \multicolumn{2}{c|}{ AlphaMax }
& \multicolumn{2}{c|}{ Elkan-Noto }
& \multicolumn{2}{c|}{ GMM }
\\
 \cline{3-8}& $\alpha$ & 100 & 1000 & 100 & 1000 & 100 & 1000 \\
  \hline
  \hline
  \multirow{5}{1\LL}{$\mathcal{N}$\\\mbox{$(\Delta \mu =1)$}}  & 0.050 & \textbf{0.154}* & \textbf{0.149}* & 0.465 & 0.463 & 0.261 & 0.378
 \\
 & 0.250 & \textbf{0.166} & 0.177 & 0.485 & 0.435 & 0.230 & \textbf{0.137}*
 \\
 & 0.500 & \textbf{0.178}* & 0.213 & 0.400 & 0.352 & 0.265 & \textbf{0.133}*
 \\
 & 0.750 & \textbf{0.201} & \textbf{0.102}* & 0.279 & 0.205 & 0.246 & 0.156
 \\
 & 0.950 & 0.262 & 0.119 & \textbf{0.126}* & \textbf{0.051}* & 0.277 & 0.165
 \\
\hline
\multirow{5}{1\LL}{$\mathcal{N}$\\\mbox{$(\Delta \mu =2)$}}  & 0.050 & 0.028 & 0.028 & 0.105 & 0.136 & \textbf{0.014}* & \textbf{0.011}*
 \\
 & 0.250 & 0.077 & 0.073 & 0.178 & 0.164 & \textbf{0.037}* & \textbf{0.019}*
 \\
 & 0.500 & 0.090 & 0.078 & 0.211 & 0.156 & \textbf{0.058}* & \textbf{0.030}*
 \\
 & 0.750 & \textbf{0.086} & 0.062 & 0.201 & 0.112 & 0.106 & \textbf{0.034}*
 \\
 & 0.950 & 0.216 & 0.050 & \textbf{0.112} & 0.037 & 0.121 & \textbf{0.036}
 \\
\hline
\multirow{5}{1\LL}{$\mathcal{N}$\\\mbox{$(\Delta \mu =4)$}}  & 0.050 & 0.004 & 0.004 & 0.011 & 0.013 & \textbf{0.002}* & \textbf{0.001}*
 \\
 & 0.250 & 0.016 & 0.005 & 0.032 & 0.018 & \textbf{0.005}* & \textbf{0.002}*
 \\
 & 0.500 & 0.021 & 0.006 & 0.063 & 0.022 & \textbf{0.007}* & \textbf{0.003}*
 \\
 & 0.750 & 0.041 & 0.013 & 0.107 & 0.016 & \textbf{0.008}* & \textbf{0.003}*
 \\
 & 0.950 & 0.171 & 0.015 & 0.103 & 0.013 & \textbf{0.017}* & \textbf{0.002}*
 \\
\hline
\hline
 \multirow{5}{1\LL}{$\mathcal{L}$\\\mbox{$(\Delta \mu =1)$}}  & 0.050 & \textbf{0.234}* & \textbf{0.261}* & 0.404 & 0.390 & 0.634 & 0.666
 \\
 & 0.250 & \textbf{0.205}* & \textbf{0.219}* & 0.404 & 0.361 & 0.379 & 0.396
 \\
 & 0.500 & \textbf{0.174}* & \textbf{0.167}* & 0.370 & 0.297 & 0.366 & 0.404
 \\
 & 0.750 & \textbf{0.224} & \textbf{0.087}* & 0.275 & 0.187 & 0.245 & 0.235
 \\
 & 0.950 & 0.467 & 0.171 & \textbf{0.123}* & \textbf{0.051}* & 0.398 & 0.186
 \\
\hline
\multirow{5}{1\LL}{$\mathcal{L}$\\\mbox{$(\Delta \mu =2)$}}  & 0.050 & 0.080 & 0.071 & 0.124 & 0.120 & \textbf{0.059} & \textbf{0.011}*
 \\
 & 0.250 & 0.086 & 0.080 & 0.164 & 0.128 & \textbf{0.057}* & \textbf{0.061}*
 \\
 & 0.500 & \textbf{0.074}* & \textbf{0.068}* & 0.198 & 0.119 & 0.235 & 0.205
 \\
 & 0.750 & \textbf{0.059}* & \textbf{0.050}* & 0.204 & 0.088 & 0.198 & 0.177
 \\
 & 0.950 & 0.229 & 0.040 & 0.127 & 0.041 & \textbf{0.106} & \textbf{0.034}
 \\
\hline
\multirow{5}{1\LL}{$\mathcal{L}$\\\mbox{$(\Delta \mu =4)$}}  & 0.050 & \textbf{0.004}* & \textbf{0.004}* & 0.016 & 0.020 & 0.015 & 0.014
 \\
 & 0.250 & 0.014 & \textbf{0.009}* & 0.040 & 0.026 & \textbf{0.010} & 0.012
 \\
 & 0.500 & 0.028 & 0.014 & 0.087 & 0.028 & \textbf{0.009}* & \textbf{0.005}*
 \\
 & 0.750 & 0.038 & 0.009 & 0.140 & 0.026 & \textbf{0.013}* & \textbf{0.003}*
 \\
 & 0.950 & 0.194 & \textbf{0.014} & 0.126 & 0.015 & \textbf{0.073}* & 0.074
 \\
\hline
\end{tabularx}
\label{tab:GL}
\normalsize
\end{table}

\autoref{tab:BW} and Figure \ref{fig:WB} show the performance over two synthetic data sets in which the true mixing proportion was systematically varied. The results suggest sensitivity of all estimation procedures for small component samples and generally good performance for large component samples. The strong performance of both Elkan-Noto and GMM-T models is similarly expected since the data sets show good separation between positive and negative examples, resulting in very high areas under the ROC curve (not shown). The results on these synthetic data sets suggest that the multivariate-to-univariate transformation did not negatively influence AlphaMax and GMM algorithms.


\begin{table}[t]
\centering
\caption{Mean absolute difference between estimated and true mixing proportion over a selection of true mixing proportions and the following data sets: $\mathcal{W}$ = waveform, and $\mathcal{B}$ = ball in the box. Statistical significance was evaluated by comparing the \AlgName\ method, the \elkan\ algorithm, and the Gaussian mixture model after applying multivariate-to-univariate transforms (GMM-T). The bold font type indicates the winner and the asterisk indicates statistical significance.}
\footnotesize
\tracingtabularx
\begin{tabularx}{0.65\linewidth}
{|>{\setlength{\hsize}{0.110000\hsize}}X|
>{\setlength{\hsize}{0.10000\hsize}}X|
>{\setlength{\hsize}{0.131667\hsize}}X|
>{\setlength{\hsize}{0.131667\hsize}}X|
>{\setlength{\hsize}{0.131667\hsize}}X|
>{\setlength{\hsize}{0.131667\hsize}}X|
>{\setlength{\hsize}{0.131667\hsize}}X|
>{\setlength{\hsize}{0.131667\hsize}}X|
}\hline 
\multirow{2}{1\LL}{Data} &
& \multicolumn{2}{c|}{ AlphaMax }
& \multicolumn{2}{c|}{ Elkan-Noto }
& \multicolumn{2}{c|}{ GMM-T }
\\
 \cline{3-8}& $\alpha$ & 100 & 1000 & 100 & 1000 & 100 & 1000 \\
  \hline
  \hline

\multirow{5}{1\LL}{$\mathcal{B}$}  & 0.050 & \textbf{0.004}* & 0.004 & 0.029 & 0.021 & 0.024 & \textbf{0.004}
 \\
 & 0.250 & \textbf{0.022} & 0.027 & 0.067 & 0.033 & 0.024 & \textbf{0.002}*
 \\
 & 0.500 & \textbf{0.027} & 0.014 & 0.110 & 0.040 & 0.036 & \textbf{0.001}*
 \\
 & 0.750 & \textbf{0.126} & 0.017 & 0.142 & 0.039 & 0.214 & \textbf{0.003}*
 \\
 & 0.950 & 0.392 & 0.030 & 0.104 & 0.024 & \textbf{0.021}* & \textbf{0.019}
 \\
\hline
\multirow{5}{1\LL}{$\mathcal{W}$}  & 0.050 & \textbf{0.004}* & \textbf{0.004}* & 0.059 & 0.059 & 0.057 & 0.147
 \\
 & 0.250 & \textbf{0.088}* & \textbf{0.038}* & 0.316 & 0.106 & 0.468 & 0.097
 \\
 & 0.500 & \textbf{0.259}* & 0.050 & 0.573 & 0.126 & 0.450 & \textbf{0.048}
 \\
 & 0.750 & 0.379 & \textbf{0.052} & 0.779 & 0.121 & \textbf{0.210}* & 0.087
 \\
 & 0.950 & 0.412 & 0.261 & 0.688 & \textbf{0.049} & \textbf{0.015}* & 0.059
 \\
\hline
\end{tabularx}
\normalsize
\label{tab:BW}
\end{table}

\newcommand{\figwidth}{\textwidth}

\begin{figure*}[]
\includegraphics[width = \figwidth]
{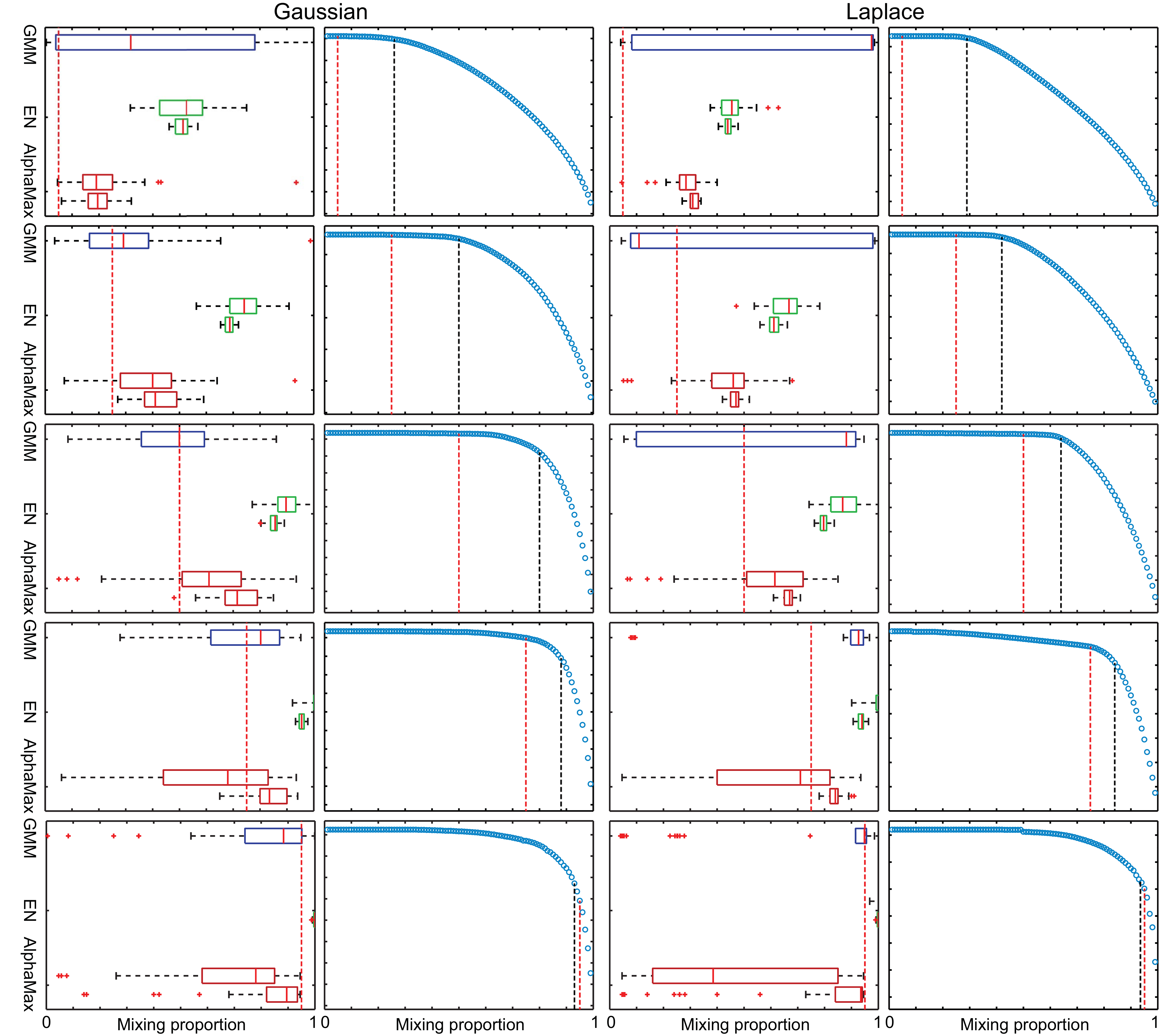}
\caption{Box plots and log-likelihood plots for the two unit-variance Gaussian distributions and two unit-scale Laplace distributions, with the separation of means $\Delta \mu = 1$. Each row corresponds to a different mixing proportion $\alpha \in \{0.05, 0.25, 0.50, 0.75, 0.95 \}$. The box plots show the performance of \AlgName, \elkan, and GMM algorithms, where \AlgName\ and \elkan\ were shown when $|X_1| = 100$ and $|X_1| = 1000$, top to bottom. The log-likelihood plots illustrate the selection of inflection points by \AlgName\ for one randomly selected example from the panel left to it. The red dashed line in each plot shows the true mixing proportion, whereas the black dashed line in the log-likelihood plots shows the mixing proportion selected in that particular example.}
\label{fig:GL} 
\end{figure*}

\begin{figure*}[]
\includegraphics[width = \figwidth]
{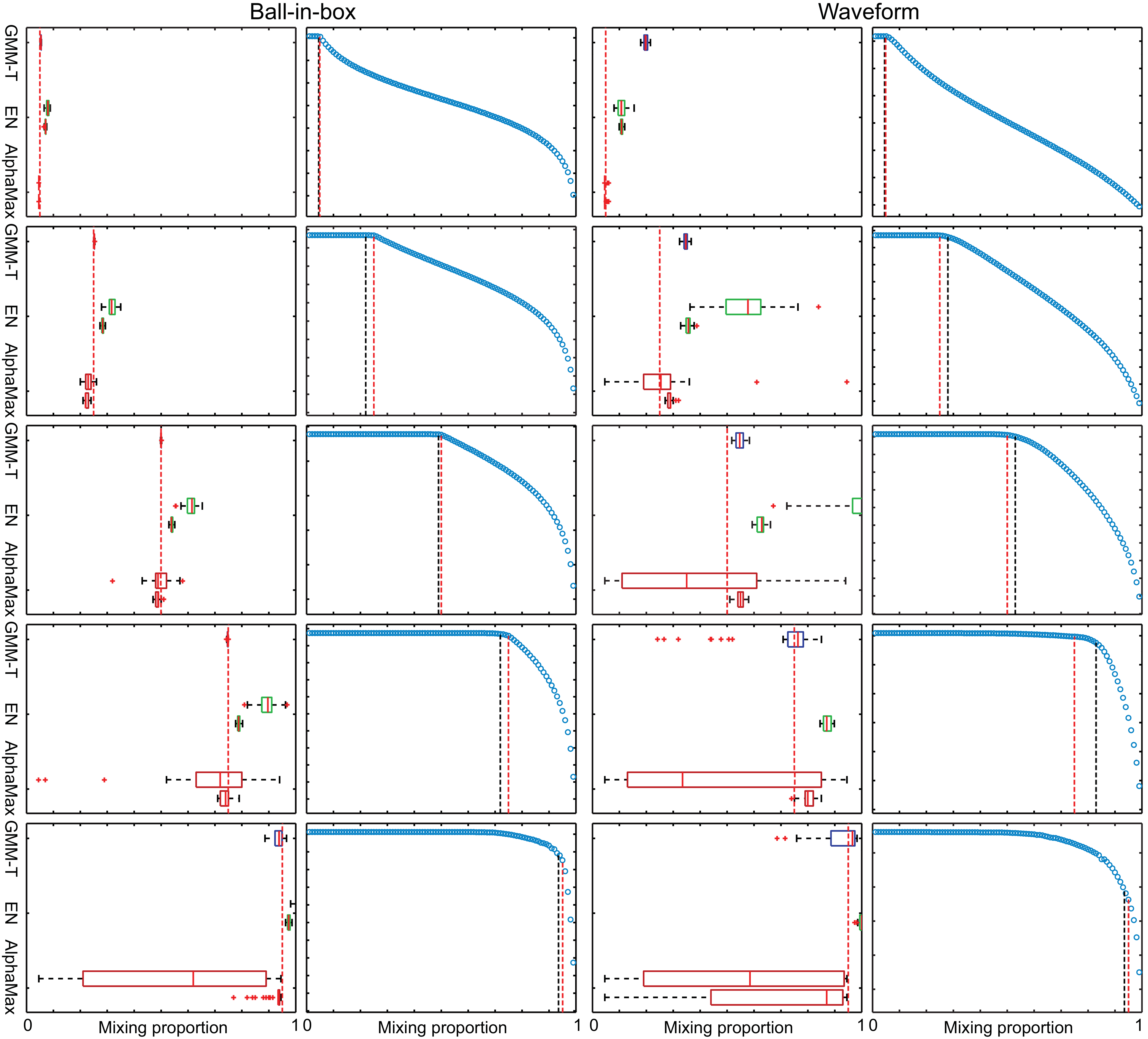}
\caption{Box plots and log-likelihood plots for the Ball-in-box data set ($d = 10$) and Waveform data set ($d = 21$). Each row corresponds to a different mixing proportion $\alpha \in \{0.05, 0.25, 0.50, 0.75, 0.95 \}$. The box plots show the performance of \AlgName, \elkan, and GMM algorithms, where \AlgName\ and \elkan\ were shown when $|X_1| = 100$ and $|X_1| = 1000$, top to bottom. The log-likelihood plots illustrate the selection of inflection points by \AlgName\ for one randomly selected example from the panel left to it. The red dashed line in each plot shows the true mixing proportion, whereas the black dashed line in the log-likelihood plots shows the mixing proportion selected in that particular example.}
\label{fig:WB} 
\end{figure*}

\begin{figure*}[]
\includegraphics[width = \figwidth]
{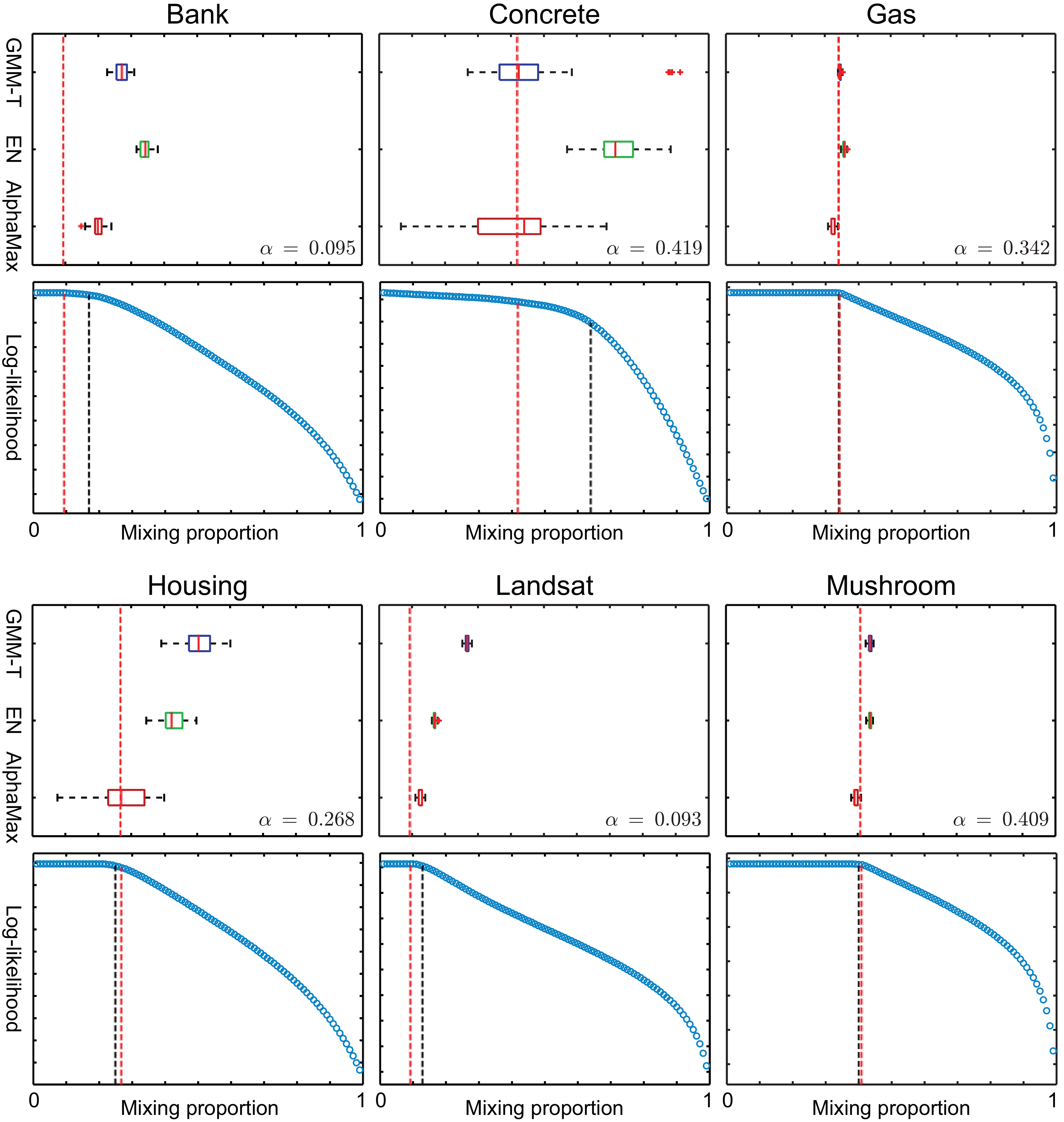}
\caption{Box plots and log-likelihood plots for six data sets from the UCI Machine Learning Repository: Bank, Concrete, Gas, Housing, Landsat, and Mushroom. The box plots show the performance of \AlgName, \elkan, and GMM-T algorithms. The log-likelihood plots illustrate the selection of inflection points by \AlgName\ for one randomly selected example from the panel above it. The red dashed line in each plot shows the true mixing proportion, whereas the black dashed line in the log-likelihood plots shows the mixing proportion selected in that particular example.}
\label{fig:RD1}
\end{figure*}

\begin{figure*}[]
\includegraphics[width = \figwidth]
{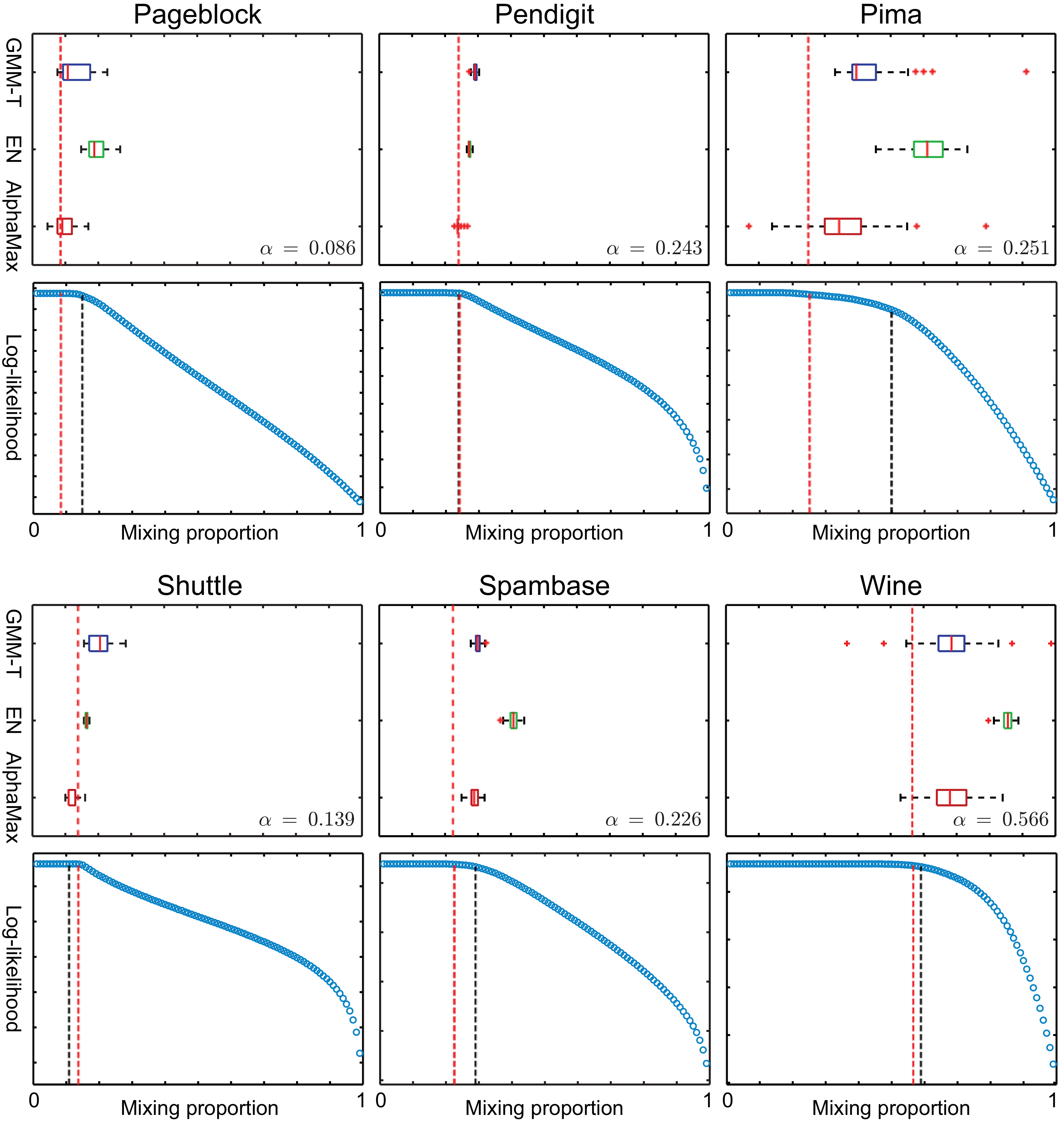}
\caption{Box plots and log-likelihood plots for six data sets from the UCI Machine Learning Repository: Pageblock, Pendigit, Pima, Shuttle, Spambase, and Wine. The box plots show the performance of \AlgName, \elkan, and GMM-T algorithms. The log-likelihood plots illustrate the selection of inflection points by \AlgName\ for one randomly selected example from the panel above it. The red dashed line in each plot shows the true mixing proportion, whereas the black dashed line in the log-likelihood plots shows the mixing proportion selected in that particular example.}
\label{fig:RD2} 
\end{figure*}


Finally, \autoref{tab:RD} shows the mean absolute error from the true mixing proportion over twelve real-life data sets from the UCI Machine Learning Repository. Here, \AlgName \ outperformed the remaining methods on ten data sets, of which eight results were statistically significant. Figures \ref{fig:RD1}-\ref{fig:RD2} additionally show the box plots for \AlgName, \elkan, and GMM-T methods, as well as the log-likelihood plots for \AlgName. All results provide evidence of the strong potential of \AlgName\ for accurate estimation of class priors in real-life situations.

\begin{table}
\caption{Mean absolute difference between estimated and true mixing proportion over twelve data sets from the UCI Machine Learning Repository. Statistical significance was evaluated by comparing the \AlgName\ method, the \elkan\ algorithm, and the Gaussian mixture model after applying multivariate-to-univariate transforms (GMM-T). The bold font type indicates the winner and the asterisk indicates statistical significance. For each data set, shown are the true mixing proportion ($\alpha$), the dimensionality of the sample ($d$), the number of positive examples ($n_1$), and the total number of examples ($n$).}
\footnotesize
\begin{centering}
\begin{tabular}{|c|c|c|c|c|c|c|c|}
\hline 
Data set & $\alpha$ & $d$ & $n_1$ & $n$ & \AlgName & Elkan-Noto & GMM-T\\
\hline 
\hline 
Bank  & 0.095 & 13 & 5188 & 45000 & \textbf{0.105}* & 0.247 & 0.177
 \\
\hline
Concrete  & 0.419 & 8 & 490 & 1030 & 0.118 & 0.299 & \textbf{0.100} 
 \\
\hline
Gas  & 0.342 & 127 & 2565 & 5574 & 0.018 & 0.017 & \textbf{0.005}*
 \\
\hline
Housing  & 0.268 & 13 & 209 & 506 & \textbf{0.062}* & 0.158 & 0.240 
 \\
\hline
Landsat  & 0.093 & 36 & 1508 & 6435 & \textbf{0.032}* &0.075 & 0.174
 \\
\hline
Mushroom  & 0.409 & 126 & 3916 & 8124 & \textbf{0.016}* & 0.028 & 0.028
 \\
\hline
Pageblock  & 0.086 & 10 & 560 & 5473 & \textbf{0.025}* & 0.109 & 0.046 
 \\
\hline
Pendigit  & 0.243 & 16 & 3430 & 10992 & \textbf{0.007}* & 0.031 & 0.048
 \\
\hline
Pima  & 0.251 & 8 & 268 & 768 & \textbf{0.129}* & 0.357 & 0.183
 \\
\hline
Shuttle  & 0.139 & 9 & 8903 & 58000 & \textbf{0.023} & 0.026 & 0.068
 \\
\hline
Spambase  & 0.226 & 57 & 1813 & 4601 & \textbf{0.066}* & 0.182 & 0.074
 \\
\hline

Wine  & 0.566 & 11 & 4113 & 6497 & \textbf{0.122} & 0.288 & 0.130
 \\
 \hline 
\end{tabular}
\normalsize
\par\end{centering}
\label{tab:RD}
\end{table}

\begin{table}[t]
\centering
\caption{Mean absolute difference between estimated and true mixing proportion over a selection of true mixing proportions and the following data sets: $\mathcal{N}$ = Gaussian with $\Delta \mu \in \{1, 2, 4\}$, $\mathcal{L}$ = Laplace with $\Delta \mu \in \{1, 2, 4\}$. Statistical significance was evaluated by comparing the \AlgName\ method, the \pdfratio\ method, and the \cdfratio\ method. The bold font type indicates the winner and the asterisk indicates statistical significance.}
\footnotesize
\tracingtabularx
\begin{tabularx}{0.70\linewidth}
{|>{\setlength{\hsize}{0.19000\hsize}}X|
>{\setlength{\hsize}{0.10000\hsize}}X|
>{\setlength{\hsize}{0.121667\hsize}}X|
>{\setlength{\hsize}{0.121667\hsize}}X|
>{\setlength{\hsize}{0.121667\hsize}}X|
>{\setlength{\hsize}{0.121667\hsize}}X|
>{\setlength{\hsize}{0.121667\hsize}}X|
>{\setlength{\hsize}{0.121667\hsize}}X|
}\hline 
\multirow{2}{1\LL}{Data} &
& \multicolumn{2}{c|}{ \AlgName }
& \multicolumn{2}{c|}{ \pdfratio }
& \multicolumn{2}{c|}{ \cdfratio }
\\
 \cline{3-8}& $\alpha$ & 100 & 1000 & 100 & 1000 & 100 & 1000 \\
  \hline
  \hline
 \multirow{5}{1\LL}{$\mathcal{N}$\\  \mbox{$(\Delta \mu =1)$}}  & 0.050 & 0.154 & 0.149 & \textbf{0.041} & \textbf{0.032}* & 0.058 & 0.042 
 \\ 
 & 0.250 & 0.166 & 0.177 & 0.116 & 0.092 & \textbf{0.089}* & \textbf{0.052}* 
 \\ 
 & 0.500 & 0.178 & 0.213 & 0.182 & 0.213 & \textbf{0.156} & \textbf{0.098}* 
 \\ 
 & 0.750 & \textbf{0.201} & \textbf{0.102}* & 0.315 & 0.308 & 0.264 & 0.176 
 \\ 
 & 0.950 & \textbf{0.262}* & \textbf{0.119}* & 0.447 & 0.449 & 0.355 & 0.202 
 \\ 
\hline 
\multirow{5}{1\LL}{$\mathcal{N}$\\  \mbox{$(\Delta \mu =2)$}}  & 0.050 & 0.028 & 0.028 & 0.026 & 0.037 & \textbf{0.014}* & \textbf{0.012}* 
 \\ 
 & 0.250 & 0.077 & 0.073 & 0.110 & 0.116 & \textbf{0.077} & \textbf{0.044}* 
 \\ 
 & 0.500 & \textbf{0.090}* & \textbf{0.078} & 0.209 & 0.172 & 0.170 & 0.081 
 \\ 
 & 0.750 & \textbf{0.086}* & \textbf{0.062}* & 0.300 & 0.300 & 0.259 & 0.169 
 \\ 
 & 0.950 & \textbf{0.216}* & \textbf{0.050}* & 0.421 & 0.380 & 0.376 & 0.186 
 \\ 
\hline 
\multirow{5}{1\LL}{$\mathcal{N}$\\  \mbox{$(\Delta \mu =4)$}}  & 0.050 & \textbf{0.004}* & \textbf{0.004}* & 0.021 & 0.031 & 0.010 & 0.012 
 \\ 
 & 0.250 & \textbf{0.016}* & \textbf{0.005}* & 0.092 & 0.092 & 0.066 & 0.038 
 \\ 
 & 0.500 & \textbf{0.021}* & \textbf{0.006}* & 0.129 & 0.085 & 0.101 & 0.049 
 \\ 
 & 0.750 & \textbf{0.041}* & \textbf{0.013}* & 0.310 & 0.268 & 0.265 & 0.145 
 \\ 
 & 0.950 & \textbf{0.171}* & \textbf{0.015}* & 0.532 & 0.372 & 0.486 & 0.227 
 \\ 
\hline
\hline
 \multirow{5}{1\LL}{$\mathcal{L}$\\  \mbox{$(\Delta \mu =1)$}}  & 0.050 & 0.234 & 0.261 & \textbf{0.077}* & \textbf{0.070}* & 0.137 & 0.169 
 \\ 
 & 0.250 & 0.205 & 0.219 & 0.116 & 0.145 & \textbf{0.087}* & \textbf{0.100}* 
 \\ 
 & 0.500 & 0.174 & 0.167 & 0.233 & 0.326 & \textbf{0.132} & \textbf{0.069}* 
 \\ 
 & 0.750 & \textbf{0.224} & \textbf{0.087}* & 0.464 & 0.524 & 0.257 & 0.158 
 \\ 
 & 0.950 & \textbf{0.467} & \textbf{0.171}* & 0.695 & 0.726 & 0.498 & 0.286 
 \\ 
\hline 
\multirow{5}{1\LL}{$\mathcal{L}$\\ \mbox{$(\Delta \mu =2)$}}  & 0.050 & 0.080 & 0.071 & \textbf{0.028}* & 0.038 & 0.035 & \textbf{0.033} 
 \\ 
 & 0.250 & 0.086 & 0.080 & 0.111 & 0.171 & \textbf{0.052}* & \textbf{0.037}* 
 \\ 
 & 0.500 & \textbf{0.074}* & \textbf{0.068}* & 0.251 & 0.379 & 0.133 & 0.098 
 \\ 
 & 0.750 & \textbf{0.059}* & \textbf{0.050}* & 0.413 & 0.550 & 0.252 & 0.194 
 \\ 
 & 0.950 & \textbf{0.229}* & \textbf{0.040}* & 0.581 & 0.735 & 0.388 & 0.286 
 \\ 
\hline 
\multirow{5}{1\LL}{$\mathcal{L}$\\ \mbox{$(\Delta \mu =4)$}}  & 0.050 & \textbf{0.004}* & \textbf{0.004}* & 0.032 & 0.044 & 0.010 & 0.008 
 \\ 
 & 0.250 & \textbf{0.014}* & \textbf{0.009}* & 0.123 & 0.172 & 0.068 & 0.044 
 \\ 
 & 0.500 & \textbf{0.028}* & \textbf{0.014}* & 0.255 & 0.282 & 0.154 & 0.117 
 \\ 
 & 0.750 & \textbf{0.038}* & \textbf{0.009}* & 0.419 & 0.469 & 0.270 & 0.175 
 \\ 
 & 0.950 & \textbf{0.194}* & \textbf{0.014}* & 0.605 & 0.639 & 0.435 & 0.272 
 \\ 
\hline
\end{tabularx}
\label{tab:GL_pdf}
\normalsize
\end{table}

\begin{table}[t]
\centering
\caption{Mean absolute difference between estimated and true mixing proportion over a selection of true mixing proportions and the following data sets: $\mathcal{W}$ = waveform, and $\mathcal{B}$ = ball in the box. Statistical significance was evaluated by comparing the \AlgName\ method, the \pdfratio\ method, and the \cdfratio\ method. The bold font type indicates the winner and the asterisk indicates statistical significance.}
\footnotesize
\tracingtabularx
\begin{tabularx}{0.65\linewidth}
{|>{\setlength{\hsize}{0.110000\hsize}}X|
>{\setlength{\hsize}{0.10000\hsize}}X|
>{\setlength{\hsize}{0.131667\hsize}}X|
>{\setlength{\hsize}{0.131667\hsize}}X|
>{\setlength{\hsize}{0.131667\hsize}}X|
>{\setlength{\hsize}{0.131667\hsize}}X|
>{\setlength{\hsize}{0.131667\hsize}}X|
>{\setlength{\hsize}{0.131667\hsize}}X|
}\hline 
\multirow{2}{1\LL}{Data} &
& \multicolumn{2}{c|}{ \AlgName }
& \multicolumn{2}{c|}{ \pdfratio }
& \multicolumn{2}{c|}{ \cdfratio }
\\
 \cline{3-8}& $\alpha$ & 100 & 1000 & 100 & 1000 & 100 & 1000 \\
  \hline
  \hline

\multirow{5}{1\LL}{$\mathcal{B}$}  & 0.050 & \textbf{0.004}* & \textbf{0.004}* & 0.044 & 0.050 & 0.018 & 0.022 
 \\ 
 & 0.250 & \textbf{0.022}* & \textbf{0.027}* & 0.121 & 0.150 & 0.068 & 0.069 
 \\ 
 & 0.500 & \textbf{0.027}* & \textbf{0.014}* & 0.175 & 0.150 & 0.103 & 0.138 
 \\ 
 & 0.750 & \textbf{0.126} & \textbf{0.017}* & 0.312 & 0.198 & 0.160 & 0.127 
 \\ 
 & 0.950 & \textbf{0.392}* & \textbf{0.030}* & 0.700 & 0.300 & 0.634 & 0.161 
 \\ 
\hline
\multirow{5}{1\LL}{$\mathcal{W}$}  & 0.050 & \textbf{0.004}* & \textbf{0.004}* & 0.050 & 0.049 & 0.017 & 0.018 
 \\ 
 & 0.250 & \textbf{0.088} & \textbf{0.038}* & 0.211 & 0.192 & 0.097 & 0.056 
 \\ 
 & 0.500 & 0.259 & \textbf{0.050} & 0.394 & 0.198 & \textbf{0.214} & 0.057 
 \\ 
 & 0.750 & 0.379 & \textbf{0.052}* & 0.565 & 0.298 & \textbf{0.337} & 0.089 
 \\ 
 & 0.950 & \textbf{0.412}* & 0.261 & 0.758 & 0.520 & 0.653 & \textbf{0.206} 
 \\ 
\hline
\end{tabularx}
\normalsize
\label{tab:BW_pdf}
\end{table}

\begin{table}
\caption{Mean absolute difference between estimated and true mixing proportion over twelve data sets from the UCI Machine Learning Repository. Statistical significance was evaluated by comparing the \AlgName\ method, the \pdfratio\ method, and the \cdfratio\ method. The bold font type indicates the winner and the asterisk indicates statistical significance. For each data set, shown are the true mixing proportion ($\alpha$), the dimensionality of the sample ($d$), the number of positive examples ($n_1$), and the total number of examples ($n$).}
\footnotesize
\begin{centering}
\begin{tabular}{|c|c|c|c|c|c|c|c|}
\hline 
Data set & $\alpha$ & $d$ & $n_1$ & $n$ & \AlgName & \pdfratio & \cdfratio \\
\hline 
\hline 
Bank  & 0.095 & 13 & 5188 & 45000 & 0.105 & 0.069 & \textbf{0.029}*
 \\
\hline
Concrete  & 0.419 & 8 & 490 & 1030 & \textbf{0.118} & 0.216 & 0.138
 \\
\hline
Gas  & 0.342 & 127 & 2565 & 5574 & \textbf{0.018}* & 0.335 & 0.137
 \\
\hline
Housing  & 0.268 & 13 & 209 & 506 & \textbf{0.062}* & 0.141 & 0.086 
 \\
\hline
Landsat  & 0.093 & 36 & 1508 & 6435 & 0.032 & 0.077 & \textbf{0.018}* 
 \\
\hline
Mushroom  & 0.409 & 126 & 3916 & 8124 &  \textbf{0.016}* & 0.318 & 0.140 
 \\
\hline
Pageblock  & 0.086 & 10 & 560 & 5473 & \textbf{0.025}* & 0.085 & 0.039 
 \\
\hline
Pendigit  & 0.243 & 16 & 3430 & 10992 & \textbf{0.007}* & 0.173 & 0.056 
 \\
\hline
Pima  & 0.251 & 8 & 268 & 768 & 0.129 & 0.116 & \textbf{0.072}* 
 \\
\hline
Shuttle  & 0.139 & 9 & 8903 & 58000 & \textbf{0.023}* & 0.122 & 0.052 
 \\
\hline
Spambase  & 0.226 & 57 & 1813 & 4601 & 0.066 & 0.116 & \textbf{0.027}* 
 \\
\hline
Wine  & 0.566 & 11 & 4113 & 6497 & 0.122 & 0.275 & \textbf{0.084}* 
 \\
 \hline 
\end{tabular}
\normalsize
\par\end{centering}
\label{tab:RD_pdf}
\end{table}

\section{Conclusions}

This work was motivated by the problem of estimating the fraction of positive examples in unlabeled data given a sample of positive examples and a sample of unlabeled data. We formulate this estimation problem as parameter learning of two-component mixture models. In this general setting, we provide theoretical analysis of the identifiability conditions and use it to develop efficient algorithms for non-parametric estimation of mixing proportions. In addition, we address the problem of density estimation in high-dimensional samples by developing class prior-preserving transformations that map the original multivariate data to univariate samples. We applied our algorithms on univariate and multivariate samples and compared them favorably with state-of-the-art supervised and unsupervised procedures.

There are several possibilities for extending this work that could lead to further improvements in the quality of estimates. These include optimization regularizers that would favor larger $\alpha$, more sophisticated kernel-density estimation, a development of better heuristics for the detection of the inflection point (e.g.~via derivative estimation) in the log-likelihood plots, as well as development of estimates of the reliability of the identified mixing proportion. Finally, extensions of this methodology to noisy and biased data will be subject of our future work.

\acks{This work was partially supported by the National Science Foundation award DBI-0644017 and National Institutes of Health award R01MH105524.}


\bibliography{refdb,additional}

\begin{thebibliography}{44}
\providecommand{\natexlab}[1]{#1}
\providecommand{\url}[1]{\texttt{#1}}
\expandafter\ifx\csname urlstyle\endcsname\relax
  \providecommand{\doi}[1]{doi: #1}\else
  \providecommand{\doi}{doi: \begingroup \urlstyle{rm}\Url}\fi

\bibitem[Beana et~al.(2013)Beana, Dimarcoa, Mercer, Thayer, Roya, and
  Ghosal]{Beana2013}
G.~J. Beana, E.~A. Dimarcoa, L.~D. Mercer, L.~K. Thayer, A.~Roya, and
  S.~Ghosal.
\newblock Finite skew-mixture models for estimation of positive false discovery
  rates.
\newblock \emph{Stat Methodol}, 10:\penalty0 46--57, 2013.

\bibitem[Blanchard et~al.(2010)Blanchard, Lee, and Scott]{Blanchard2010}
G.~Blanchard, G.~Lee, and C.~Scott.
\newblock Semi-supervised novelty detection.
\newblock \emph{J Mach Learn Res}, 11:\penalty0 2973--3009, 2010.

\bibitem[Breiman et~al.(1984)Breiman, Friedman, Olshen, and Stone]{Breiman1984}
L.~Breiman, J.~H. Friedman, R.~A. Olshen, and C.~J. Stone.
\newblock \emph{Classification and regression trees}.
\newblock Wadsworth International Group, Belmont, CA, 1984.

\bibitem[Chawla et~al.(2004)Chawla, Japkowicz, and Kotcz]{Chawla2004}
N.~V. Chawla, N.~Japkowicz, and A.~Kotcz.
\newblock Editorial: special issue on learning from imbalanced data sets.
\newblock \emph{ACM SIGKDD Explorations Newsletter}, 6\penalty0 (1):\penalty0
  1--6, 2004.

\bibitem[Cooley and MacEachern(1998)]{cooley1998classification}
C.~A. Cooley and S.~N. MacEachern.
\newblock Classification via kernel product estimators.
\newblock \emph{Biometrika}, 85\penalty0 (4):\penalty0 823--833, 1998.

\bibitem[Cortes et~al.(2008)Cortes, Mohri, Riley, and Rostamizadeh]{Cortes2008}
C.~Cortes, M.~Mohri, M.~Riley, and A.~Rostamizadeh.
\newblock Sample selection bias correction theory.
\newblock In \emph{Proceedings of the 19th International Conference on
  Algorithmic Learning Theory}, ALT 2008, pages 38--53. Springer, 2008.

\bibitem[Dempster et~al.(1977)Dempster, Laird, and Rubin]{Dempster1977}
A.~P. Dempster, N.~M. Laird, and D.~B. Rubin.
\newblock Maximum likelihood from data via the {EM} algorithm.
\newblock \emph{J R Statist Soc B}, 39\penalty0 (1):\penalty0 1--38, 1977.

\bibitem[Denis(1998)]{Denis1998}
F.~Denis.
\newblock {PAC} learning from positive statistical queries.
\newblock In \emph{Proceedings of the 9th International Conference on
  Algorithmic Learning Theory}, ALT 1998, pages 112--126, 1998.

\bibitem[Denis et~al.(2005)Denis, Gilleron, and Letouzey]{Denis2005}
F.~Denis, R.~Gilleron, and F.~Letouzey.
\newblock Learning from positive and unlabeled examples.
\newblock \emph{Theor Comput Sci}, 348\penalty0 (16):\penalty0 70--83, 2005.

\bibitem[Dessimoz et~al.(2013)Dessimoz, Skunca, and Thomas]{Dessimoz2013}
C.~Dessimoz, N.~Skunca, and P.~D. Thomas.
\newblock {CAFA} and the open world of protein function predictions.
\newblock \emph{Trends Genet}, 29\penalty0 (11):\penalty0 609--610, 2013.

\bibitem[du~Plessis and Sugiyama(2012)]{duPlessis2012}
M.~C. du~Plessis and M.~Sugiyama.
\newblock Semi-supervised learning of class balance under class-prior change by
  distribution matching.
\newblock In \emph{Proceedings of the 29th International Conference on Machine
  Learning}, ICML 2012, pages 823--830, 2012.

\bibitem[du~Plessis and Sugiyama(2014)]{duPlessis2014b}
M.~C. du~Plessis and M.~Sugiyama.
\newblock Class prior estimation from positive and unlabeled data.
\newblock \emph{IEICE Transactions on Information and Systems}, E97-D\penalty0
  (5):\penalty0 1358--1362, 2014.

\bibitem[du~Plessis et~al.(2014)du~Plessis, Niu, and Sugiyama]{duPlessis2014}
M.~C. du~Plessis, G.~Niu, and M.~Sugiyama.
\newblock Analysis of learning from positive and unlabeled data.
\newblock In \emph{Advances in Neural Information Processing Systems}, NIPS
  2014, pages 703--711. Curran Associates, Inc., 2014.

\bibitem[Elkan(2001)]{Elkan2001}
C.~Elkan.
\newblock The foundations of cost-sensitive learning.
\newblock In \emph{Proceedings of the 17th International Joint Conference on
  Artificial Intelligence}, IJCAI 2001, pages 973--978, 2001.

\bibitem[Elkan and Noto(2008)]{Elkan2008}
C.~Elkan and K.~Noto.
\newblock Learning classifiers from only positive and unlabeled data.
\newblock In \emph{Proceedings of the 14th ACM SIGKDD International Conference
  on Knowledge Discovery and Data Mining}, KDD 2008, pages 213--220, New York,
  NY, USA, 2008. ACM.

\bibitem[Geurts(2011)]{Geurts2011}
P.~Geurts.
\newblock Learning from positive and unlabeled examples by enforcing
  statistical significance.
\newblock In \emph{Proceedings of the 14th International Conference on
  Artificial Intelligence and Statistics}, AISTATS 2011, pages 305--314, 2011.

\bibitem[Ghosal and Roy(2011)]{Ghosal2011}
S.~Ghosal and A.~Roy.
\newblock Identifiability of the proportion of null hypotheses in skew-mixture
  models for the p-value distribution.
\newblock \emph{Electron J Statist}, 5:\penalty0 329--341, 2011.

\bibitem[Hastie et~al.(2001)Hastie, Tibshirani, and Friedman]{Hastie2001}
T.~Hastie, R.~Tibshirani, and J.~H. Friedman.
\newblock \emph{The elements of statistical learning: data mining, inference,
  and prediction}.
\newblock Springer Verlag, New York, NY, 2001.

\bibitem[Heckman(1979)]{Heckman1979}
J.~Heckman.
\newblock Sample selection bias as a specification error.
\newblock \emph{Econometrica}, 47:\penalty0 153--161, 1979.

\bibitem[Latinne et~al.(2001)Latinne, Saerens, and Decaestecker]{Latinne2001}
P.~Latinne, M.~Saerens, and C.~Decaestecker.
\newblock Adjusting the outputs of a classifier to new a priori probabilities
  may significantly improve classification accuracy: evidence from a
  multi-class problem in remote sensing.
\newblock In C.~E. Brodley and A.~P. Danyluk, editors, \emph{Proceedings of the
  18th International Conference on Machine Learning}, ICML 2001, pages
  298--305. Morgan Kaufmann, 2001.

\bibitem[Lee and Liu(2003)]{Lee2003}
W.~S. Lee and B.~Liu.
\newblock Learning with positive and unlabeled examples using weighted logistic
  regression.
\newblock In \emph{Proceedings of the 20th International Conference on Machine
  Learning}, ICML 2003, pages 448--455, 2003.

\bibitem[Lichman(2013)]{Lichman2013}
M.~Lichman.
\newblock {UCI Machine Learning Repository}, 2013.
\newblock URL \url{http://archive.ics.uci.edu/ml}.

\bibitem[Liu et~al.(2002)Liu, Lee, Yu, and Li]{Liu2002}
B.~Liu, W.~S. Lee, P.~S. Yu, and X.~Li.
\newblock Partially supervised classification of text documents.
\newblock In \emph{Proceedings of the 19th International Conference on Machine
  Learning}, ICML 2002, pages 387--394, 2002.

\bibitem[Liu et~al.(2003)Liu, Dai, Li, Lee, and Yu]{Liu2003}
B.~Liu, Y.~Dai, X.~Li, W.S. Lee, and P.~S. Yu.
\newblock Building text classifiers using positive and unlabeled examples.
\newblock In \emph{Proceedings of the 3rd IEEE International Conference on Data
  Mining}, ICDM 2003, pages 179--186, 2003.

\bibitem[Liu et~al.(2007)Liu, Lafferty, and Wasserman]{liu2007sparse}
H.~Liu, J.~D. Lafferty, and L.~A. Wasserman.
\newblock Sparse nonparametric density estimation in high dimensions using the
  rodeo.
\newblock In \emph{Proceedings of the 11th International Conference on
  Artificial Intelligence and Statistics}, AISTATS 2007, pages 283--290, 2007.

\bibitem[Manevitz and Yousef(2001)]{Manevitz2001}
L.~M. Manevitz and M.~Yousef.
\newblock One-class {SVMs} for document classification.
\newblock \emph{J Mach Learn Res}, 2:\penalty0 139--154, 2001.

\bibitem[McLachlan and Peel(2000)]{McLachlan2000}
G.~J. McLachlan and D.~Peel.
\newblock \emph{Finite mixture models}.
\newblock John Wiley \& Sons, Inc., New York, NY 10158, USA, 2000.

\bibitem[Park and Sandberg(1991)]{Park1991}
J.~Park and I.~W. Sandberg.
\newblock Universal approximation using radial-basis-function networks.
\newblock \emph{Neural Comput}, 3\penalty0 (2):\penalty0 246--257, 1991.

\bibitem[Pearson(1900)]{Pearson1900}
K.~Pearson.
\newblock On the criterion that a given system of deviations from the probable
  in the case of the correlated systems of variables is such that it can be
  reasonably supposed to have arisen from random sampling.
\newblock \emph{Philos Magazine 5th Series}, 50\penalty0 (302):\penalty0
  157--175, 1900.

\bibitem[Phillips et~al.(2009)Phillips, Dudik, Elith, Graham, Lehmann,
  Leathwick, and Ferrier]{Phillips2009}
S.~J. Phillips, M.~Dudik, J.~Elith, C.~H. Graham, A.~Lehmann, J.~Leathwick, and
  S.~Ferrier.
\newblock Sample selection bias and presence-only distribution models:
  implications for background and pseudo-absence data.
\newblock \emph{Ecol Appl}, 19\penalty0 (1):\penalty0 181--197, 2009.

\bibitem[Platt(1999)]{Platt1999}
J.~C. Platt.
\newblock \emph{Probabilistic outputs for support vector machines and
  comparison to regularized likelihood methods}, pages 61--74.
\newblock MIT Press, 1999.

\bibitem[Saerens et~al.(2002)Saerens, Latinne, and Decaestecker]{Saerens2002}
M.~Saerens, P.~Latinne, and C.~Decaestecker.
\newblock Adjusting the outputs of a classifier to new a priori probabilities:
  a simple procedure.
\newblock \emph{Neural Comput}, 14:\penalty0 21--41, 2002.

\bibitem[Scott and Blanchard(2009)]{Scott2009}
C.~Scott and G.~Blanchard.
\newblock Novelty detection: unlabeled data definitely help.
\newblock In \emph{Proceedings of the 12th International Conference on
  Artificial Intelligence and Statistics}, AISTATS 2009, pages 464--471, 2009.

\bibitem[Scott and Nowak(2005)]{Scott2005}
C.~Scott and R.~Nowak.
\newblock A {Neyman-Pearson} approach to statistical learning.
\newblock \emph{IEEE Trans Inf Theory}, 51\penalty0 (11):\penalty0 3806--3819,
  2005.

\bibitem[Scott et~al.(2013)Scott, Blanchard, and Handy]{Scott2013}
C.~Scott, G.~Blanchard, and G.~Handy.
\newblock Classification with asymmetric label noise: consistency and maximal
  denoising.
\newblock \emph{J Mach Learn Res W\&CP}, 30:\penalty0 489--511, 2013.

\bibitem[Scott(1979)]{Scott1979}
D.~W. Scott.
\newblock On optimal and data-based histograms.
\newblock \emph{Biometrika}, 66\penalty0 (3):\penalty0 605--610, 1979.

\bibitem[Scott(2008)]{scott2008curse}
D.~W. Scott.
\newblock The curse of dimensionality and dimension reduction.
\newblock \emph{Multivariate Density Estimation: Theory, Practice, and
  Visualization}, pages 195--217, 2008.

\bibitem[Storey and Tibshirani(2003)]{Storey2003}
J.~D. Storey and R.~Tibshirani.
\newblock Statistical significance for genomewide studies.
\newblock \emph{Proc Natl Acad Sci U S A}, 100\penalty0 (16):\penalty0
  9440--9445, 2003.

\bibitem[Tallis and Chesson(1982)]{Tallis1982}
G.~M. Tallis and P.~Chesson.
\newblock Identifiability of mixtures.
\newblock \emph{J Austral Math Soc Ser A}, 32:\penalty0 339--348, 1982.

\bibitem[Vucetic and Obradovic(2001)]{Vucetic2001}
S.~Vucetic and Z.~Obradovic.
\newblock Classification on data with biased class distribution.
\newblock In \emph{Proceedings of the 12th European Conference on Machine
  Learning}, ECML 2001, pages 527--538, 2001.

\bibitem[Ward et~al.(2009)Ward, Hastie, Barry, Elith, and Leathwick]{Ward2009}
G.~Ward, T.~Hastie, S.~Barry, J.~Elith, and J.R. Leathwick.
\newblock Presence-only data and the {EM} algorithm.
\newblock \emph{Biometrics}, 65\penalty0 (2):\penalty0 554--563, 2009.

\bibitem[Yakowitz and Spragins(1968)]{Yakowitz1968}
S.~J. Yakowitz and J.~D. Spragins.
\newblock On the identifiability of finite mixtures.
\newblock \emph{Ann Math Statist}, 39:\penalty0 209--214, 1968.

\bibitem[Yu et~al.(2004)Yu, Han, and Chang]{Yu2004}
H.~Yu, J.~Han, and K.~C.~C. Chang.
\newblock {PEBL}: web page classification without negative examples.
\newblock \emph{IEEE Trans Knowl Data Eng}, 16\penalty0 (1):\penalty0 70--81,
  2004.

\bibitem[Zhang and Lee(2005)]{Zhang2005}
D.~Zhang and W.~S. Lee.
\newblock A simple probabilistic approach to learning from positive and
  unlabeled examples.
\newblock In \emph{Proceedings of the 5th Annual UK Workshop on Computational
  Intelligence}, UKCI 2005, pages 83--87, 2005.

\end{thebibliography}

\end{document}